\documentclass[11pt]{article}

\usepackage[letterpaper, portrait, margin=1in]{geometry}
\usepackage[colorlinks=true,linkcolor=blue,citecolor=ForestGreen]{hyperref}
\usepackage{algorithm}
\usepackage[noend]{algpseudocode}
\usepackage{booktabs,multirow}
\usepackage{url}
\usepackage{amsmath,amssymb,amsthm}
\usepackage[noabbrev,capitalise,nameinlink]{cleveref}
\usepackage{mathtools}
\usepackage{parskip}
\usepackage{xspace}
\usepackage{verbatim}
\usepackage{mathrsfs}
\usepackage[usenames,dvipsnames,svgnames,table]{xcolor}
\usepackage{pgf}
\usepackage[dvipsnames]{xcolor}
\usepackage{todo}
\usepackage{tabularx}

\usepackage{dsfont}

\theoremstyle{plain}
\newtheorem{theorem}{Theorem}[section]
\newtheorem{lemma}[theorem]{Lemma}

\newtheorem{proposition}[theorem]{Proposition}

\newtheorem{corollary}[theorem]{Corollary}

\theoremstyle{definition}
\newtheorem{definition}[theorem]{Definition}
\newtheorem{remark}[theorem]{Remark}
\newtheorem{example}[theorem]{Example}

\theoremstyle{plain}

\newcommand{\ignore}[1]{}

\DeclareMathOperator{\supp}{supp}   
\DeclareMathOperator{\poly}{poly}
\DeclareMathOperator{\sign}{sign}

\newcommand{\dist}{\mathsf{dist}}

\newcommand{\mlearn}{m^{\mathrm{learn}}}
\newcommand{\mlearne}{\mlearn_\epsilon}
\newcommand{\mtest}{m^{\mathrm{test}}}
\newcommand{\mteste}{\mtest_\epsilon}
\newcommand{\mone}{m^{\mathrm{one}}}
\newcommand{\qtol}{q^{\mathrm{test}}}
\newcommand{\qtole}{\qtol_{\epsilon_0,\epsilon_1}}
\newcommand{\qtolz}{\qtol_{0,\epsilon}}

\newcommand{\Ex}[1]{\bE \left[ #1 \right]}

\renewcommand{\Pr}[1]{\bP \left[ #1 \right]} 
\newcommand{\Pru}[2]{\underset{ #1 }\bP \left[ #2 \right]}
\newcommand{\Pruc}[3]{\underset{ #1 }\bP \left[ #2 \;\; \mid \;\; #3 \right]}

\newcommand{\define}{\vcentcolon=}

\renewcommand{\exp}[1]{\mathrm{exp}\left( #1 \right)}

\newcommand{\ind}[1]{\mathds{1} \left[ #1 \right] }

\newcommand{\zo}{\{0,1\}}
\newcommand{\pmset}{\{\pm 1\}}

\newcommand{\note}[1]{{\color{Brown} \;\ifmmode \text{#1} \else #1 \fi\;}}

\newcommand{\cA}{\ensuremath{\mathcal{A}}}
\newcommand{\cB}{\ensuremath{\mathcal{B}}}
\newcommand{\cC}{\ensuremath{\mathcal{C}}}
\newcommand{\cD}{\ensuremath{\mathcal{D}}}
\newcommand{\cF}{\ensuremath{\mathcal{F}}}

\newcommand{\cH}{\ensuremath{\mathcal{H}}}
\newcommand{\cI}{\ensuremath{\mathcal{I}}}

\newcommand{\cL}{\ensuremath{\mathcal{L}}}
\newcommand{\cM}{\ensuremath{\mathcal{M}}}

\newcommand{\cP}{\ensuremath{\mathcal{P}}}
\newcommand{\cQ}{\ensuremath{\mathcal{Q}}}

\newcommand{\cS}{\ensuremath{\mathcal{S}}}
\newcommand{\cT}{\ensuremath{\mathcal{T}}}

\newcommand{\cX}{\ensuremath{\mathcal{X}}}

\newcommand{\bE}{\ensuremath{\mathbb{E}}}

\newcommand{\bN}{\ensuremath{\mathbb{N}}}
\newcommand{\bP}{\ensuremath{\mathbb{P}}}
\newcommand{\bR}{\ensuremath{\mathbb{R}}}

\newcommand{\bZ}{\ensuremath{\mathbb{Z}}}

\usepackage{float}
\usepackage{afterpage}
\restylefloat{table}

\newcommand{\VC}{\mathsf{VC}}
\newcommand{\LVC}{\mathsf{LVC}}
\newcommand{\TV}{\mathsf{TV}}
\newcommand{\DEC}{\mathsf{DEC}}
\newcommand{\SSD}{\mathsf{SSD}}

\newcommand{\Poi}{\mathsf{Poi}}
\newcommand{\sh}{\mathsf{sh}}

\creflabelformat{equation}{#2\textup{#1}#3}

\title{VC Dimension and Distribution-Free Sample-Based Testing}
\date{}
\author{Eric Blais\thanks{Research supported by an NSERC Discovery grant.} \\ University of Waterloo \\ 
        \texttt{eric.blais@uwaterloo.ca}
\and Renato Ferreira Pinto, Jr. \\ 
        \texttt{r4ferrei@uwaterloo.ca}
\and Nathaniel Harms\thanks{Research supported by an NSERC Canada Graduate Scholarship.} \\ University of Waterloo \\ 
        \texttt{nharms@uwaterloo.ca}}

\begin{document}
\maketitle

\begin{abstract}
We consider the problem of determining which classes of functions can be tested more efficiently
than they can be learned, in the distribution-free sample-based model that corresponds to the standard PAC learning setting. Our main result shows that while VC dimension by itself does not always provide tight bounds on the number of samples required to test a class of functions in this model, it can be combined with a closely-related variant that we call ``lower VC'' (or LVC) dimension to obtain strong lower bounds on this sample complexity. 

We use this result to obtain strong and in many cases nearly optimal bounds on the sample complexity for testing unions of intervals, halfspaces, intersections of halfspaces, polynomial threshold functions, and decision trees. Conversely, we show that two natural classes of functions, juntas and monotone functions, can be tested with a number of samples that is polynomially smaller than the number of samples required for PAC learning.

Finally, we also use the connection between VC dimension and property testing to establish new lower bounds for testing radius clusterability and testing feasibility of linear constraint systems. 
\end{abstract}

\thispagestyle{empty}
\setcounter{page}{0}
\newpage

\section{Introduction}

For which classes of functions can we test membership in the class more efficiently than we can
learn a good approximation of a function in the class? This is a central question in property
testing that was initially posed in the seminal work of Goldreich, Goldwasser, \& Ron~\cite{GGR98}
and Kearns \& Ron \cite{KR00}, and has since received a considerable amount of attention in
different models of property testing and learning. In the standard PAC learning model of
Valiant~\cite{Valiant84}, the learner is \emph{sample-based} (using only random examples of the
function) and \emph{distribution-free} (it must work for any distribution, unknown to the learner).
\cite{GGR98} introduce distribution-free sample-based testers and ``stress that [this model] is
essential for some of the potential applications'' listed in that paper, but despite much recent
interest in both distribution-free and sample-based testing
(e.g.~\cite{GS09,BBBY12,AHW16,GR16,CFSS17,BMR19,BY19,Har19,FY20,RR20}), even basic questions for
this model remain unanswered. For example, are halfspaces more efficiently testable than learnable?

More precisely, fix a set $\cH$ of Boolean-valued functions over some domain $\cX$.  There is an
unknown distribution $\cD$ over $\cX$, the distance between two functions $f,g : \cX \to \{0,1\}$ is
$\dist_{\cD}(f,g) = \Pru{x \sim \cD}{f(x) \neq g(x)}$, and the distance between $f$ and $\cH$ is
$\inf_{h \in \cH} \dist_{\cD}(f,h)$.  In both the learning and testing problems, the algorithm is
given a set of $m$ labelled examples $(x,f(x))$ with each $x$ drawn independently from $\cD$.  For
some fixed $\epsilon > 0$, the goals of the algorithms are:
\begin{quote}
\begin{description}
\item[Learning:] When $f \in \cH$, output a function $h$ that satisfies $\dist_\cD(f,h) \le
\epsilon$;
\item[Testing:] Accept when $f \in \cH$ and reject when $\dist_\cD(f,\cH) \ge \epsilon$.
\end{description}
\end{quote}
In both cases, the algorithms are required to satisfy the condition with probability at least
$\frac23$ (in this paper, we study testing algorithms with two-sided error). Let $\mlearne(\cH)$ and
$\mteste(\cH)$ denote the minimum sample complexity of a learning and testing algorithm for $\cH$,
respectively. Except in pathological cases (such as $\cH$ being a singleton), $\mteste(\cH) =
O(\mlearne(\cH))$ \cite{GGR98}.\footnote{The definitions above correspond to the standard
$(\epsilon,\delta)$-PAC learning definition with $\delta = \frac13$ and to distribution-free
sample-based property testing, respectively. Note that for property testing over fixed (and known)
distributions, the upper bound on sample complexity holds for \emph{proper} (not general) learning
sample complexity~\cite{GGR98}.} The main question can now be phrased as: 
\begin{quote}
\emph{For which classes $\cH$ of
Boolean-valued functions is $\mteste(\cH) \ll \mlearne(\cH)$?}
\end{quote}

The fundamental result of PAC learning (see e.g.~\cite{SSBD14}) is that the VC dimension of $\cH$
determines $\mlearne(\cH)$.  Recall that a set $T \subseteq \cX$ is \emph{shattered} by $\cH$ if for
every $\ell : T \to \{0,1\}$ there is a function $f \in \cH$ that agrees with $\ell$ on all points
in $T$.  The \emph{VC dimension of $\cH$ with respect to $S \subseteq \cX$} is
\[
\VC_S(\cH)
  \define \max \{ k : \exists \text{$T \subseteq S$ of size $|T| = k$ that is shattered by $\cH$}\}.
\]
(When $S = \cX$ we will often omit the subscript and write simply $\VC(\cH)$.) When $\epsilon > 0$
is constant, $\mlearne(\cH) = \Theta(\VC(\cH))$, so to understand the relationship between
$\mteste(\cH)$ and $\mlearne(\cH)$, it is necessary to understand the relationship between
$\mteste(\cH)$ and the VC dimension.

The VC dimension has appeared in the property testing literature, but its use has been limited to
upper bounds (e.g.~\cite{GGR98,ADPR03,BBBY12,AFZ19}, see however the recent work of Livni \&
Mansour~\cite{LM19} which shows some lower bounds for \emph{graph-based discrimination} between
distributions, a problem related to testing properties of distributions, in terms of a VC-like
notion of dimension); the relationship mentioned above implies $\mteste(\cH) = O(\VC(\cH))$ for
constant $\epsilon$. This paper is concerned with finding \emph{lower bounds} in terms of the VC
dimension. Such lower bounds would be desirable not only for understanding the relationship between
testing and learning, but also because they would be \emph{combinatorial} in nature, obtained via an
analysis of the structure of the function class, whereas nearly all known lower bounds in
sample-based property testing (e.g.,~\cite{GGR98,KR00,BBBY12,BY19,RR20}) are \emph{distributional}:
a probability distribution specific to the problem is constructed and shown to be hard to test.

It is clear that VC dimension cannot, in general, be a lower bound on the sample complexity of
testing: consider the following example from \cite{GGR98}. Let $\cH$ be the set of all Boolean
functions $f : \zo^n \to \zo$ that satisfy $f(x)=1$ for all $x \in \zo^n$ with $x_1=1$.  $\VC(\cH) =
\Theta(2^n)$ since the $2^{n-1}$ points $x$ with $x_1=0$ are shattered, while $\mteste(\cH) =
O(1/\epsilon)$.  Therefore, the relationship of VC dimension to (distribution-free sample-based)
property testing is more complicated than to (PAC) learning, and we must
introduce some new ideas.

\subsection{Our results}

The central message of the current work is that the VC dimension \emph{can} give lower bounds
on $\mteste(\cH)$ when it is combined with a closely-related combinatorial measure.
For any class $\cH$ of Boolean-valued functions over $\cX$ and any subset $S \subseteq \cX$,
define the \emph{LVC dimension} (or \emph{Lower Vapnik-Chervonenkis dimension}) of 
$\cH$ with respect to $S$ to be
\[
\LVC_S(\cH)
  \define \max \{ k : \forall \text{$T \subseteq S$ of size $|T| = k$, 
    $T$ is shattered by $\cH$}\}.
\]
The definition of LVC dimension differs from that of the VC dimension only by the replacement 
of the existential quantifier with a universal one. 
This immediately implies that $\LVC(\cH) \le \VC(\cH)$ for every class $\cH$ and motivates 
our choice to call this measure ``lower'' VC dimension.
And in some cases, the LVC dimension of a class can be 
much smaller than its VC dimension. (See \cref{section:sauer lemma} for a discussion of some
concepts in learning theory related to LVC dimension.)
Our main theorem gives a general lower bound on $\mteste(\cH)$ in terms of the VC and LVC
dimensions of $\cH$.

\begin{theorem}
\label{thm:main-informal}
There is a constant $C > 0$ such that 
for any class $\cH$ of Boolean-valued functions over $\cX$ and any $S \subseteq \cX$, 
if $|S| > 5\,\VC_S(\cH)$ and $\LVC_S(\cH) \geq C \cdot \VC_S(\cH)^{3/4}\sqrt{\log \VC_S(\cH)}$,
then for small enough values of $\epsilon$,
\[
\mteste(\cH) = \Omega\left(\frac{\LVC_S(\cH)^2}{\VC_S(\cH)\log\VC_S(\cH)}\right).
\]
Moreover, this bound is tight as there are
classes $\cH$ for which $\mteste(\cH) = \Theta\left(\frac{\LVC_S(\cH)^2}{\VC_S(\cH)\log\VC_S(\cH)}\right)$.
\end{theorem}

For many natural classes $\cH$ of functions, there is a set $S$ where $\LVC_S(\cH) = \VC_S(\cH) = \tilde{\Omega}(\VC(\cH))$.
For these classes, the following direct consequence of
\cref{thm:main-informal} is most convenient.

\begin{corollary}
\label{cor:main-informal}
For every class $\cH$ of Boolean-valued functions over $\cX$, 
if there is a set $S \subseteq \cX$ for which
$\LVC_S(\cH) = \VC_S(\cH)$ and $|S| \geq 5 \cdot \VC_S(\cH)$, then
\[
\mteste(\cH) = \Omega\left( \frac{\VC_S(\cH)}{\log \VC_S(\cH)} \right).
\]
\end{corollary}

We use \cref{thm:main-informal} and \cref{cor:main-informal} to establish sample complexity lower
bounds for distribution-free sample-based testing of many natural classes of functions, which shows,
essentially, that two-sided testing is not significantly more efficient than learning (and, with
\cref{lemma:one-sided upper bound}, that two-sided testing is not significantly more efficient
than one-sided testing). The main new lower bounds we obtain are as follows, with the VC dimensions
of each class included for comparison.  (For the formal definitions of each class, see the section devoted to that class.)

\begin{center}
  \begin{tabular}{ll|ll|l}
    Domain 
    & Class $\cH$ 
    & $\mteste(\cH)$ 
    & 
    & $\VC(\cH)$ 
    \\ \hline
    
    $[n]$ or $\bR$ \rule{0pt}{4ex}
    & Unions of $k$ intervals 
    & $\Omega\left(\frac{k}{\log k}\right)$
    & \cref{thm:intervals} 
    & $\Theta(k)$ 
    \\
    
    $\bR^n$ \rule{0pt}{4ex}
    & Halfspaces 
    & $\Omega\left(\frac{n}{\log n}\right)$
    & \cref{thm:real halfspaces}
    & $\Theta(n)$ 
    \\

    & Intersections of $k$ halfspaces 
    & $\Omega\left(\frac{nk}{\log(nk)}\right)$
    & \cref{thm:intersections of halfspaces}
    & $\Theta(nk\log k)$ 
    \\

    & Degree-$k$ PTFs over $\bR^n$
    & $\Omega\left(\frac{\binom{n + k}{k}}
                        {\log\binom{n+k}{k}}\right)$
    & \cref{thm:analytic dudley}
    & $\Theta\left(\binom{n+k}{k}\right)$ 
    \\

    & Size-$k$ decision trees 
    & $\Omega\left(\frac{k}{\log k}\right)$
    & \cref{thm:real decision trees}
    & $\Omega(k)$
    \\
    
    $\{0,1\}^n$ \rule{0pt}{4ex}
    & Halfspaces 
    & $\Omega\left(\frac{n}{\log n}\right)$
    & \cref{thm:boolean halfspaces}
    & $\Theta(n)$ 
    \\

    & Degree-$k$ PTFs
    & $\Omega\left(\frac{(n/4ek)^{k}}{ k\log(n/k)}\right)$
    & \cref{thm:boolean ptfs}
    & $\Theta\left(\binom{n}{\leq k}\right)$ 
    \\

    & Size-$k$ decision trees
    & $\Omega\left(\frac{k}{\log k \cdot \log\log k}\right)$
    & \cref{thm:boolean decision trees}
    & $\Omega(k), O(k\log n)$
\end{tabular}
\end{center}

We discuss these lower bounds in more detail. For standard definitions in property testing and
learning, see the Glossary in \cref{glossary}.

\begin{description}
\item[Unions of $k$ intervals.] 
\cite{BBBY12} (see also~\cite{KR00,Neeman14}) showed that there is an algorithm that can test unions
of $k$ intervals over \emph{any} distribution on $[0,1]$ with only $O(\sqrt{k})$ samples---as long
as the distribution is known to the algorithm. Our lower bound for this class shows that the sample
must be quadratically larger if the distribution is not known to the algorithm.

Our bound also has implications for the \emph{active testing} model \cite{BBBY12}, where a tester
can draw some unlabelled samples from the unknown distribution $\cD$ and then query the value of the
target function on any of the sampled points. Blum and Hu~\cite{BH18} showed that it is possible to
tolerantly test unions of $k$ intervals in this model with $O(k)$ samples and $O(1)$ queries.
\cref{thm:intervals} implies that $\widetilde \Omega(k)$ samples are necessary, even for intolerant
active testers (regardless of how many samples are queried), so their result is essentially optimal.

\item[Halfspaces.]
When testing over the Gaussian distribution on $\bR^n$, only $O(\sqrt{n})$ samples suffice to test
halfspaces~\cite{BBBY12}; in fact, $\widetilde O(\sqrt n)$ samples suffice in the ``partially
distribution-free'' setting where the distribution is unknown but promised to be rotation-invariant
\cite{Har19}. (With query access to the function, only a \emph{constant} number of queries are
required to test halfspaces over the Gaussian distribution or the uniform distribution on the
hypercube~\cite{MORS10}.) Epstein \& Silwal \cite{ES20,Sil20} show that $\Omega(d/\epsilon)$ samples
are required for testing the class of halfspaces over $\bR^n$ with one-sided error. (See
\cref{remark:ES}.) Our lower bound establishes a quadratic gap in sample complexity between
rotation-invariant and general distribution-free testing, and the lower bound holds even for the
hypercube $\zo^n$.

\item[Intersections of halfspaces, and polynomial threshold functions (PTFs).]
Intersections of halfspaces~\cite{BEHW89,CMK19} and polynomial threshold
functions~\cite{KS04,HellersteinS07,DHK+10,OS10} have received much attention in the learning theory
literature, but very few bounds are known on the sample or query complexity for testing these
classes. As far as we know, the only bound known for testing intersections of $k$ halfspaces is an
upper bound of $\exp{k \log k}$ queries for testing the class over the Gaussian
distribution~\cite{DeMN19} and no bound is known for testing polynomial threshold functions of
degree greater than 1. So our results appear to establish the first non-trivial lower bounds
specific for either of these classes in any model of property testing.

\item[Decision trees.]
Kearns and Ron~\cite{KR00} first studied the problem of
testing size-$k$ decision trees, showing that $\Omega(\sqrt{k})$ samples
are necessary to test the class over the uniform distribution and that this bound can be matched in the 
parameterized property testing model where the algorithm 
must only distinguish size-$k$ decision trees from functions that are far from size-$k'$ decision trees 
over the uniform distribution for
some $k' > k$. The sample complexity of the (non-parameterized) size-$k$ decision tree testing problem
over the uniform distribution is not known. 
(The query complexity for testing size-$k$ decision trees is also far from settled: despite recent notes to the contrary in~\cite{Saglam18,Bsh20}, the best current lower bound for the query complexity of testing size-$k$ decision trees is $\Omega(\log k)$~\cite{CGM11,Tan20}; see also~\cite{BlaisBM12} for a stronger lower bound for testers with one-sided error.)
\end{description}

Our techniques can also be used to establish lower bounds for other models of testing.
First, we show an application to testing properties of sets of
points---properties that correspond to unsupervised learning problems. Namely, the \emph{radius
clustering} problem can be represented by the class $\cC_k$ that consists of all sets of points $X
\subseteq \bR^n$ that can be covered by the union of at most $k$ unit-radius balls.  A distribution
$\cD$ on $\bR^n$ is \emph{$k$-clusterable} if its support is in $\cC_k$, and it is
\emph{$\epsilon$-far from $k$-clusterable} if the total variation distance between $\cD$ and any
$k$-clusterable distribution is at least $\epsilon$.  Alon et al.~\cite{ADPR03} showed that
$O(\tfrac{1}{\epsilon}nk \log(nk))$ samples from $\cD$ suffice to $\epsilon$-test
$k$-clusterability---that is, to distinguish $k$-clusterable distributions from those that are
$\epsilon$-far from $k$-clusterable. (This can be improved to
$O(\tfrac{1}{\epsilon}nk\log(k)\log\tfrac{1}{\epsilon})$; see~\cite{H14} and \cref{section:clustering}.)
Prior to this work, the only lower bound for the sample complexity of this problem was Epstein and
Silwal's recent lower bound of $\Omega(d/\epsilon)$ samples for $\epsilon$-testing 1-clusterability
with one-sided error~\cite{ES20,Sil20}.
(See \cref{remark:ES}.)
We give a lower bound for two-sided error testers that is tight up
to poly-log factors for all values of $k$ up to $2^{n/6}$.

\begin{theorem}
\label{thm:clustering intro}
For sufficiently small constant $\epsilon > 0$, any 
two-sided $\epsilon$-tester for $k$-clusterability in $\bR^n$ 
must have sample complexity $\Omega\big( \frac{nk}{\log(nk)} \big)$.
\end{theorem}

A variant of \cref{thm:main-informal} can also be used to prove strong
lower bounds for some testing problems even when the underlying distribution is guaranteed to be uniform
over an unknown subset of the domain. Such a situation occurs in the recent model of testing LP-type
problems introduced by Epstein and Silwal~\cite{ES20}; see \cref{section:lp} for the
details. We show that in this model, testing feasibility of a linear program with $n$ variables and
two-sided error requires $n^{1-o(1)}$ queries, almost matching the $O(n/\epsilon)$ upper bound of
\cite{ES20}.

The connection between LVC dimension and distribution-free property testing also extends to the
tolerant testing model even when the algorithm has \emph{query} access to the function and can
adaptively select the queries during its execution. The analogue of \cref{thm:main-informal}
in that setting is as follows.

\begin{theorem}
\label{thm:main-tolerant}
There is a constant $C > 0$ such that 
for any class $\cH$ of Boolean-valued functions over $\cX$ and any $S \subseteq \cX$, 
if $|S| > 5\VC_S(\cH)$ and $\LVC_S(\cH) \geq C \cdot \VC_S(\cH)^{3/4}\sqrt{\log \VC_S(\cH)}$,
then for small enough values of $\epsilon_1 > 0$ and for all $0 \le \epsilon_0 < \epsilon_1$,
the number $\qtole(\cH)$ required to $(\epsilon_0,\epsilon_1)$-test $\cH$ is bounded below by
\[
\qtole(\cH) = \Omega\left(\frac{\LVC_S(\cH)^2}{\VC_S(\cH)\log(\VC_S(\cH))}\right).
\]
\end{theorem}

The bound in the theorem applies even to the tolerant testing model where $\epsilon_0 = 0$ (i.e.,
the algorithm must accept functions that are $0$-close to the class) but it does \emph{not} apply to
the non-tolerant testing model: using \cref{thm:main-tolerant}, we show in \cref{section:upper} that
there are classes $\cH$ where $\qtolz(\cH)$ is nearly linear in the size of the domain but
non-tolerant distribution-free testing can be accomplished (even non-adaptively) with only
$O(1/\epsilon^2)$ queries.

Finally, we question the necessity of the conditions in \cref{thm:main-informal}. We have proved
that many commonly-studied classes of functions meet the conditions of the theorem and therefore are
impossible to test much more efficiently than learn; but are there commonly-studied classes of
functions that fail the condition in the theorem and have efficient distribution-free sample-based
testers?  \cite{GGR98} gave an example a class where distribution-free sample-based learning is much
more efficient than learning, which we repeated above, but this is not a commonly-studied, natural
class. In \cref{section:juntas} and \cref{section:monotonicity} we prove that two foundational
properties in the property testing literature, $k$-juntas and monotone Boolean functions, have
distribution-free sample-based testers with complexity $O(\VC^c)$ for constants $c < 1$.

\subsection{Our techniques}

Our main tool is a reduction from testing properties of \emph{distributions} to testing properties
of \emph{functions}. Some relationships between these two types of problems have been observed
before, e.g.~by Sudan \cite{Sud10} and Goldreich \& Ron \cite{GR16}, who note that any distribution
testing problem can be reduced to a testing problem for a specially-constructed symmetric property
of functions with non-Boolean range (symmetric properties are invariant under permutations of the
variables). Goldreich \& Ron \cite{GR16} also observed that testing symmetric properties of
functions  can be reduced to \emph{support-size estimation}, a fundamental problem in distribution
testing. We extend this connection between distribution testing and property testing to
properties that are not symmetric: we show, in the opposite direction of \cite{GR16}, that
support-size testing can be reduced to property testing when the LVC dimension is large.

The generality of our lower bound comes from the
Sauer--Shelah--Perles lemma, which is usually used to prove \emph{upper bounds} in terms of the VC dimension. However, we use it to show that a random function is far from the property
$\cH$ when the underlying distribution has support size larger than the VC dimension. On the other
hand, informally, random functions are indistinguishable from functions in $\cH$ when the underlying
distribution is supported on a set smaller than the LVC dimension. In this way, we can show that
any testing algorithm must implicitly solve the support-size testing problem by distinguishing
between distributions of small vs.~large supports. Tight bounds on the support-size estimation
problem were attained by Valiant \& Valiant \cite{VV11a,VV11b}; we use a version of the bound due to
Wu \& Yang \cite{WY19} which applies to a wider range of parameters that are necessary for our
reduction.

With this technique, the problem of attaining lower bounds is transformed into the combinatorial
problem of constructing appropriate sets $S$ with large $\LVC_S(\cH)$ and $\VC_S(\cH)$. We study
this problem in the second half of the paper. For some properties, like halfspaces, this is easy,
but other properties require more effort:

\paragraph{Intersections of halfspaces.}
The VC dimension of intersections of $k$ halfspaces in $\bR^n$ is $\Theta(nk\log k)$ in general
\cite{CMK19}.  We construct a subset $S \subset \bR^n$ in which $\LVC_S = \VC_S = \Theta(nk)$. We
accomplish this by constructing $S$ such that intersections of $k$ halfspaces on $S$ are equivalent
to $\Theta(nk)$-alternating functions on $\bR$.  This reduction yields a lower bound of
$\Omega(nk/\log(nk))$ for testing intersections of $k$ halfspaces.

\paragraph{Polynomial threshold functions on $\bR^n$.}
Although PTFs can be transformed into halfspaces in a higher dimension, we opt to treat PTFs as a
special case of a \emph{Dudley class} \cite{BL98} and explore the connection between LVC dimension,
Dudley classes, and \emph{maximum classes}. Dudley classes are those obtained by taking the sign of
a function in a fixed vector space $\cF$ of real-valued functions, and these classes have VC
dimension equal to the dimension of that vector space \cite{WD81}. Maximum classes are those for
which the Sauer--Shelah--Perles lemma is tight (see \cref{section:sauer lemma}), and in
particular, those classes satisfy $\LVC=\VC$ (\cref{prop:maximum condition}). Johnson
\cite{Joh14} showed that Dudley classes with domain $\bR^n$ where the functions $\cF$ are analytic
are maximum on an arbitrarily large subset $S \subseteq \bR^n$, and therefore $\LVC_S=\VC_S$, so our
main result applies. Other examples of analytic Dudley classes include balls in $\bR^n$ and trigonometric
polynomial threshold functions in $\bR^2$, so we automatically obtain lower bounds for these classes.

\paragraph{Halfspaces and PTFs on the Boolean hypercube.}
Our constructions of the set $S$ for halfspaces and PTFs on domain $\bR^n$ fail on the more
restrictive domain $\zo^n$, and indeed the deterministic reduction underlying \cref{thm:main}
seems to fail as well, because it is hard or impossible to construct large sets $S \subset \zo^n$
with high $\LVC_S$ (observe that $\zo^n$ is far from being in general position). Therefore we use a
\emph{randomized} reduction that requires some results on the non-singularity of random matrices. In
particular, we rely on a theorem of Abbe, Shpilka, \& Wigderson \cite{ASW15} to construct a random
set $S \subseteq \pmset^n$ on which the condition $\LVC = \Omega(\VC)$ holds ``with high
probability''\!\!, i.e.~a \emph{random} set of size $\Omega(\VC)$ is shattered.

\paragraph{\texorpdfstring{$k$}{k}-Clusterability.}
$k$-Clusterability is a property of \emph{distributions}, not functions, so our main theorem does
not apply; however, we can adapt the argument to this setting. In the case $k=1$, we use
concentration results for random points on the $n$-sphere to show that a random set of $n$ points
(with $\|x\|_2 > 1$) is 1-clusterable while a random set of $2n$ points is far from $1$-clusterable,
which we can then use in a randomized reduction from support-size testing to $1$-clusterability. We
extend these concentration results to $k$ disjoint spheres to attain the randomized reduction to
$k$-clusterability.

\paragraph{Uniform distributions and testing LP feasibility.} In some cases it is be desirable to
show lower bounds for testing problems where the underlying distribution is promised to be uniform
over some unknown subset of the domain. An example is the recent model of Epstein \& Silwal
\cite{ES20} for LP-type testing. In some cases, by replacing the support-size testing lower bounds
of Wu \& Yang \cite{WY19} with the lower bounds for the \emph{distinct elements} problem
\cite{RRSS09}, we can reproduce the main theorem with a slightly weaker bound, but with the
guarantee that the distributions are uniform (over an unknown support).

\section{General Lower Bound}
\label{section:lower bound}

We prove our main result, \cref{thm:main-informal}, in this section. Before we begin, we discuss
some examples that illuminate why the conditions in the theorem are important, i.e.~the choice of a
subset $S \subseteq \cX$ with large $\LVC_S(\cH)$ and the condition $|S| \geq 5\cdot \VC_S(\cH)$.
Unlike a learning algorithm, a property tester can halt and reject as soon as it sees proof that the
unknown function $f : \cX \to \zo$ does not belong to the class. Therefore, we aim to find subsets
$S \subseteq \cX$ where small ``certificates'' of non-membership cannot exist. This motivates the
definition of $\LVC_S(\cH)$: any subset $T \subseteq S$ of size $|T| \leq \LVC_S(\cH)$ cannot
contain any certificates of non-membership, for any function $f \notin \cH$. So we want to find sets
where $\LVC_S(\cH)$ is as large as possible relative to $\VC(\cH)$. On the other hand, if
$\LVC_S(\cH)=\VC(\cH)$ but $|S|=\VC(\cH)$, then the class $\cH$ restricted to $S$ is trivial: it
contains all possible functions on $S$, so testing is still easy. $|S|$ must be large enough so that
most functions are far from $\cH$, and this will be guaranteed in general when $|S| >
5 \cdot \VC_S(\cH)$ (the constant 5 is somewhat arbitrary). The following examples illustrate these
phenomena.  In the first example, $\LVC_\cX(\cH)$ is constant, but a careful choice of large $S$
allows $\LVC_S(\cH)=\VC(\cH)$, and we will obtain lower bounds for this class:
\begin{example}
Let $\cL_n$ be the set of halfspaces $\bR^n \to \pmset$. As is well-known, $\VC_{\bR^n}(\cL_n) =
n+1$. But $\LVC_{\bR^n}(\cL_n) = 2$, since any 3 colinear points cannot be shattered. On the other
hand, if $S \subseteq \bR^n$ is a set of points in general position and $|S| > n+1$, then
$\LVC_S(\cL_n)=\VC_S(\cL_n) = n+1$.
\end{example}
In the second example, the conditions of our theorem fail: finding a good set $S$ is impossible, and indeed there is an efficient distribution-free sample-based tester; see \cref{thm:monotonicity upper bound}.

\begin{example}
Let $\cM$ be the set of monotone functions $P \to \zo$ where $P$ is any partial order ($f : P \to
\zo$ is monotone if $f(x) \leq f(y)$ whenever $x < y$). Recall that an \emph{antichain} is a set of
points $x \in P$ that are incomparable. Observe that a set $T$ is shattered by $\cM$ if and only if
it is an antichain: a monotone function can take arbitrary values on an antichain, whereas if $x,y
\in T$ are comparable, say $x < y$, then $f(x) \leq f(y)$ so $T$ cannot be shattered.  Therefore
$\LVC_S(\cM) = \VC_S(\cM) = |S|$ if $S$ is an antichain, and if $S$ is not an antichain then
$\LVC_S(\cM) = 2$ while $\VC_S(\cM)$ is the size of the largest antichain in $S$.
\end{example}

We now turn to the proof of \cref{thm:main-informal,thm:main-tolerant}. 
The proof uses two main ingredients: lower bounds on the support size estimation problem, and the
Sauer--Shelah--Perles theorem.

\subsection{Ingredient 1: Support size distinction}

A fundamental problem in the field of distribution testing is \emph{support size estimation}: Given
sample access to an unknown finitely-supported distribution $\cD$ where each element occurs with
probability at least $1/n$ (for some $n$), estimate the size of the support up to an additive
$\epsilon n$ error.  Valiant \& Valiant \cite{VV11a,VV11b} showed that for constant $\epsilon$, the
number of samples required for this problem is $\Theta\left(\frac{n}{\log n}\right)$. We will adapt
this lower bound (in fact an improved version of Wu and Yang \cite{WY19}) to give lower bounds on
distribution-free property testing. 

\begin{definition}[Support-Size Distinction Problem]
For any $n \in \bN$ and $0 < \alpha < \beta \leq 1$, define $\SSD(n,\alpha,\beta)$ as the minimum
number $m \in \bN$ such that there exists an algorithm that for any input distribution $p$ over
$[n]$, takes $m$ samples from $p$ and distinguishes with probability at least $2/3$ between the
cases:
\begin{enumerate}
\item $|\supp(p)| \leq \alpha n$ and $\forall i \in \supp(p), p_i \geq 1/n$; and,
\item $|\supp(p)| \geq \beta n$ and $\forall i \in \supp(p), p_i \geq 1/n$.
\end{enumerate}
\end{definition}
Valiant \& Valiant \cite{VV11a} and Wu \& Yang \cite{WY19} each prove lower bounds on support-size
\emph{estimation} and they do so essentially by proving lower bounds on support-size distinction. We
note that the bound of \cite{VV11a} holds for $\SSD(n,\alpha,\beta)$ when $1/2 < \alpha < \beta <
1$, but this gap of at most $1/2$ is not sufficient for our purposes, so we use the improved version
of \cite{WY19}.  However, their lower bound on $\SSD$ is not stated explicitly, and therefore we
state and prove the following bound explicitly in \cref{section:wy}.
\begin{theorem}[\cite{WY19}]
\label{thm:wy easy application}
There exists a constant $C$ such that, for any $\delta \geq C \frac{\sqrt{\log n}}{n^{1/4}}$ and
$\alpha, 1-\beta \geq \delta$,
\[
  \SSD(n,\alpha,\beta) = \Omega\left(\frac{n}{\log n}\log^2\frac{1}{1-\delta}\right) \,.
\]
\end{theorem}

\subsection{Ingredient 2: Sauer--Shelah--Perles Lemma}

We will need the Sauer--Shelah--Perles lemma (see e.g.~\cite{SSBD14}), for which we recall the following definitions:
\begin{definition}
Let $\cH$ be a set of functions $\cX \to \zo$ and let $S \subseteq \cX$.
We will define the \emph{shattering number} as
\[
  \sh(\cH,S) \define |\{ T \subseteq S \;|\; T \text{ is shattered by } \cH \}| \,.
\]
We define the \emph{growth function} as
\[
  \Phi(\cH,S) \define |\{ \ell : S \to \zo \;|\; \exists h \in \cH \; \forall x \in S,
\ell(x)=h(x)\}| \,.
\]
\end{definition}
We state a version of the Sauer--Shelah--Perles lemma that follows from the so-called
Sandwich Theorem, rediscovered by numerous authors (see e.g.~\cite{Mor12}):
\begin{lemma}[Sauer--Shelah--Perles]
Let $\cH$ be a class of functions $\cX \to \zo$ and let $S \subseteq \cX$ with $\VC_S(\cH) = d$.
Then $\Phi(\cH,S) \leq \sh(\cH,S) \leq \sum_{i=0}^d { \binom{|S|}{i} }$.
\end{lemma}
This lemma gives us a bound on the probability that a random function over a large set is far from
the hypothesis class $\cH$.
\begin{lemma}
\label{lemma:random labelling}
There are constants $L > 0, K > 1$ and $\epsilon_0 > 0$ such that, for all
$\epsilon < \epsilon_0$, if $\cH$ is a class of functions $\cX \to \zo$ with
$\VC(\cH) = d$ and $T \subseteq \cX$ has size $|T| \geq Kd$, then a uniformly random labelling
$\ell : T \to \zo$ satisfies, with probability at least $1-e^{-Ld}$, $\forall h \in \cH : \Pru{x
\sim T}{h(x) \neq \ell(x)} > \epsilon$. (In particular, $K=3.04$ suffices.)
\end{lemma}
\begin{proof}
For any $T \subseteq \cX$ of size $|T|=m$, and
each $h \in \cH$, the number of functions $\ell : T \to \zo$ that differ from $h$ on at most
$\epsilon m$ points of $T$ is at most $\sum_{i=1}^{\epsilon m} \binom{m}{i}$. Therefore, by the
Sauer-Shelah-Perles lemma, the number of labellings $\ell : T \to \zo$ that differs on at most
$\epsilon m$ points from the closest $h \in \cH$ is at most
\[
  \left(\sum_{i=0}^d \binom{m}{i} \right) \cdot \left(\sum_{i=0}^{\epsilon m} { \binom{m}{i} }
\right)
  \leq \left(\frac{em}{d}\right)^d \cdot \left(\frac{em}{\epsilon m}\right)^{\epsilon m} 
  = \left(\frac{em}{d}\right)^d \cdot \left(\frac{e}{\epsilon}\right)^{\epsilon m} \,.
\]
The probability that a uniformly random $\ell : T \to \zo$ satisfies this condition is therefore at
most
\[
  \left(\frac{em}{d}\right)^d \cdot \left(\frac{e}{\epsilon}\right)^{\epsilon m} \cdot 2^{-m}
  = (Ke)^d (e/\epsilon)^{K\epsilon d} 2^{-Kd}
  = 2^{d \left(\log(Ke)+K\epsilon\log(e/\epsilon)-K\right)}
  = e^{d \left(\ln(Ke)+K\epsilon\ln(e/\epsilon)-K\ln(2)\right)} \,,
\]
For any $K > 1$ satisfying $K\ln(2) > 1+\ln(K)$, there is $L>0,\epsilon_0 > 0$ such
that the exponent $\ln(Ke) + K\epsilon\ln(e/\epsilon)-K\ln(2) < -Ld$ for all $\epsilon <
\epsilon_0$.
\end{proof}

\subsection{Main reduction}
We now present the main reduction for the proof of \cref{thm:main-informal}. This reduction is
inspired by a proof in the recent work of Epstein \& Silwal \cite{ES20}.  The reduction can be
described intuitively as follows. Suppose there is a class $\cH$ of functions $\cX \to \zo$ and a
set $S \subseteq \cX$ such that are two thresholds $t_1 < t_2$ where:
\begin{enumerate}
\item Any set $T \subset S$ of size $|T| \leq t_1$ is shattered by $\cH$; and,
\item A random function on any subset $T \subset S$ of size $|T| \geq t_2$ is far from $\cH$ with
high probability.
\end{enumerate}
Then a distribution-free tester must accept any function (with high probability) when the
distribution has support size at most $t_1$, and reject a random function (with high probability)
when the distribution has support size at least $t_2$. This is made formal in our main lemma:
\begin{lemma}
\label{lemma:ssd reduction}
Let $\cH$ be a set of functions $\cX \to \zo$. Suppose $S \subseteq \cX$ has size $|S|=n$ and $0 <
\alpha < \beta \leq 1$ satisfy the following conditions:
\begin{enumerate}
\item $\forall T \subset S$ such that $|T| \leq \alpha n$, $T$ is shattered by $\cH$; and,
\item $\forall T \subseteq S$ such that $|T| \geq \beta n$, a uniformly random labelling $\ell : T
\to \zo$ satisfies with probability at least $9/10$ the condition
\[
  \forall h \in \cH : \Pru{x \sim T}{\ell(x) \neq h(x)} \geq \epsilon/\beta \,.
\]
\end{enumerate}
Then $\mteste(\cH), \qtolz(\cH) = \Omega(\SSD(n,\alpha,\beta))$.
\end{lemma}
\begin{proof}
Let $f : S \to \zo$ be a uniformly random function, let $\phi : [n] \to S$ be any bijection, and let
$\cD$ be any distribution over $[n]$ with $\cD(x) \geq 1/n$ for all $x \in \supp(\cD)$. Write
$\phi\cD$ for the distribution over $S$ of $\phi(x)$ when $x \sim \cD$. We make two claims.

First, if $\cD$ has support size at most $\alpha n$ then $\dist_{\phi\cD}(f,\cH)=0$. Let $T =
\supp(\phi\cD)$. Then since $|T| \leq \alpha n$, by the first condition there exists $h \in \cH$
such that $h(x) = f(x)$ on all $x \in T$. So $\dist_{\phi\cD}(f,h)=0$.

Second, if $\cD$ has support size at least $\beta n$ then with probability at least $9/10$ over the
choice of $f$, $\dist_{\phi\cD}(f,\cH)\geq \epsilon$.  Let $T = \supp(\phi\cD)$ and for any $h \in
\cH$ write $\Delta(f,h) = \{ x \in T : f(x) \neq h(x)\}$. Since $|T| \geq \beta n$ we have by
assumption that, with probability at least $9/10$ over the choice of $f$, for uniform $x \sim T$,
$\Pr{x \in \Delta(f,h)} \geq \epsilon/\beta$. Therefore $|\Delta(f,h)| \geq
\frac{\epsilon}{\beta}|T| \geq \epsilon n$. Since $\phi\cD(x) = \cD(\phi^{-1}(x)) \geq 1/n$ for
every $x \in T$, this means that for every $h \in \cH,\Pru{x \sim \phi\cD}{f(x) \neq h(x)} \geq
\frac{1}{n}|\Delta(f,h)| \geq \epsilon$.

\textbf{Sample-based testing.} We first prove the lower bound on distribution-free sample
testing. Assume there is a distribution-free tester $A$ that uses $m$ samples. The algorithm for
support-size distinction is as follows. Let $\phi : [n] \to S$ be a bijection. Given input
distribution $\cD$ over $[n]$, choose a uniformly random $f : S \to \zo$, draw $m$ samples $Q =
(x_1, \dotsc, x_m)$ from $\phi\cD$ and let $Q_f = ((x_1,f(x_1)), \dotsc, (x_m,f(x_m)))$; run $A$ on
the samples $Q_f$ and accept $\cD$ iff $A$ outputs 1.

First suppose that $\cD$ has support size at most $\alpha n$. There exists a function $h \in \cH$
with $\dist_{\phi\cD}(f,h)=0$, so $f(x)=h(x)$ for all $x \in \supp(\phi\cD)$. Therefore the samples
$Q_f$ and $Q_h$ have the same distribution, and the algorithm must output 1 on $Q_h$ with
probability at least $5/6$, so it must output 1 on $Q_f$, and therefore accept $\cD$, with
probability at least $5/6$.

Next suppose that $\cD$ has support size at least $\beta n$. Then the uniformly random function $f :
S \to \zo$ is $\epsilon$-far from $\cH$ with respect to $\phi\cD$ with probability at least $9/10$.
Assuming this occurs, algorithm $A$ must output 0 with probability at least $5/6$, so $\cD$ is
rejected with probability at least $2/3$. We conclude
\[
  \mteste(\cH) = \Omega(\SSD(n,\alpha,\beta)) \,.
\]
\textbf{Adaptive tolerant testing.}
Let $A$ be an adaptive $(0,\epsilon)$-tolerant tester for $\cH$ and assume it requests at most $m$
samples and $q$ queries. The algorithm for support-size distinction is as follows. Let $\ell : \cX
\to \zo$ be a uniformly random function $S \to \zo$ and let $f : \cX \to \zo$ be any function
agreeing with $\ell$ on $S$. Draw $m$ samples $Q = (x_1, \dotsc, x_m)$ from $\cD$ and let $Q_f =
((x_1,f(x_1)), \dotsc, (x_m,f(x_m)))$. When $A$ requests its $i^\mathrm{th}$ sample, give it
$(x_i,f(x_i))$. When $A$ queries $x \in \cX$, give it $f(x)$. Accept $\cD$ if and only if $A$
outputs 1.

If $\cD$ has support size at most $\alpha n$, then $\dist_{\phi\cD}(f,\cH)=0$, so $A$ will accept
with probability at least $5/6$. On the other hand, if $\cD$ has support size at least $\beta n$
then $\dist_{\phi\cD}(f,\cH) \geq \epsilon$ with probability at least $5/6$ over the choice of $f$,
and in this case $A$ will reject with probability at least $5/6$, so our algorithm will reject with
probability at least $2/3$. Therefore we have solved support-size distinction using only $m$
samples, and $m \leq m+q \leq \qtolz(\cH)$, so we conclude
\[
  \qtolz(\cH) = \Omega(\SSD(n,\alpha,\beta)) . \qedhere
\]
\end{proof}

\subsection{Proof of the main lower bound}

Combining \cref{thm:wy easy application} with \cref{lemma:random labelling}, we obtain
the most general form of our main theorem:
\begin{theorem}
\label{thm:main}
Let $\cH$ be a class of functions $\cX \to \zo$ and suppose there is a set $S
\subseteq \cX$ and a value $\delta \in (0,1/2)$ such that, for $n = |S|$, the following hold:
\begin{enumerate}
\item $K \cdot \VC_S(\cH) \leq (1-\delta) n$, where $K$ is the constant from \cref{lemma:random
labelling}; and,
\item $\LVC_S(\cH) \geq \delta n$; and,
\item $\delta \geq C \frac{\sqrt{ \log n}}{n^{1/4}}$ where $C$ is the constant from
\cref{thm:wy easy application}.
\end{enumerate}
Let $d = \VC_S(\cH)$. Then for some constant $\epsilon_0 > 0$ and all $0 < \epsilon < \epsilon_0$,
\[
\mteste(\cH), \qtolz(\cH)
  = \Omega\left(\frac{n}{\log n} \log^2\frac{1}{1-\delta}\right) \,.
\]
\end{theorem}
\begin{proof}
Let $\alpha = \frac{1}{n} \LVC_S(\cH), \beta = \frac{1}{n} K\cdot \VC_S(\cH)$, so that $\alpha \geq
\delta$ and $\beta \leq 1-\delta$. Then from \cref{thm:wy easy application},
\[
  \SSD(n,\alpha,\beta) = \Omega\left(\frac{n}{\log n} \log^2 \frac{1}{1-\delta}\right) \,.
\]
By definition of $\LVC$, any set $T \subseteq S$ with $|T| \leq \alpha n$ satisfies condition 1 of
\cref{lemma:ssd reduction}, and by \cref{lemma:random labelling}, any set $T \subseteq S$ such
that $|T| \geq \beta n = K\cdot \VC_S(\cH)$ satisfies condition 2 for sufficiently small
(constant) $\epsilon > 0$, so by \cref{lemma:ssd reduction},
\[
  \mteste(\cH), \qtolz(\cH)
    = \Omega(\SSD(n,\alpha,\beta))
    = \Omega\left(\frac{n}{\log n}\log^2\frac{1}{1-\delta}\right) \,.
\]
Finally, since $\frac{1}{1-\delta} \geq 1$ and $n = \Omega(d/(1-\delta)) = \Omega(d)$, we have a
lower bound of $\Omega\left(\frac{d}{\log d}\log^2 \frac{1}{1-\delta}\right)$.
\end{proof}

The following simplified bound proves  \cref{thm:main-informal,thm:main-tolerant}
 from the introduction and will also be used in most of our applications.

\begin{corollary}
\label{cor:easy bound}
There is a constant $L > 0$ such that the following holds.  Let $S \subseteq \cX$ satisfy $n \define
|S| \geq 5 \cdot \VC_S(\cH)$. If $\LVC_S(\cH)
> L \cdot \VC_S(\cH)^{3/4}\sqrt{\log\VC_S(\cH)}$, then
\[
  \mteste(\cH),\qtolz(\cH) = \Omega\left(\frac{\LVC_S(\cH)^2}{\VC_S(\cH) \log \VC_S(\cH)}\right) \,.
\]
\end{corollary}
\begin{proof}
We may assume $n=|S|=5\cdot \VC_S(\cH)$ since by taking subsets of a set $S$ of size larger than $\VC_S(\cH)$, we do not decrease the LVC dimension and
do not increase the VC dimension; we can choose a subset that also does not decrease the VC dimension. 
We may set $K=4$ in \cref{thm:main}. Let $\delta =
\frac{\LVC_S(\cH)}{2K \VC_S(\cH)}$, so
\[
  \delta n = \frac{\LVC_S(\cH)}{2K\VC_S(\cH)} \cdot 5\VC_S(\cH) \leq \LVC_S(\cH) \,.
\]
We also have $(1-\delta) n \geq \left(1-\frac{1}{8}\right) 5 \cdot \VC_S(\cH) \geq 4 \VC_S(\cH) =K \cdot \VC_S(\cH)$.
Finally,
\[
\delta = \frac{\LVC_S(\cH)}{8 \VC_S(\cH)}
    \geq \frac{L \sqrt{\log\VC_S(\cH)}}{8 \VC_S(\cH)^{1/4}}
    = \frac{5^{1/4} L \sqrt{\log(n/5)}}{8 n^{1/4}} \,,
\]
so for large enough constant $L > 0$ this is at least $C \frac{\sqrt{\log n}}{n^{1/4}}$ for the
constant $C$ in \cref{thm:wy easy application}, so the conditions for \cref{thm:main}
are satisfied, and we obtain a lower bound of
\[
  \Omega\left(\frac{\VC_S(\cH)}{\log \VC_S(\cH)}\log^2\frac{1}{1-\delta}\right) \,.
\]
Finally, using the inequality $\log^2\frac{1}{1-\delta} \geq \log^2(e^\delta) = \Omega(\delta^2)$ we
get the conclusion.
\end{proof}

\section{Geometric Classes}

In this section, we use \cref{cor:main-informal}
to prove lower bounds on the number of samples required
to test unions of intervals, halfspaces, and intersections
of halfspaces.

\emph{Technical note:} For the domain $\bR^n$, the tester may assume that the distribution $\cD$ is
defined on the same $\sigma$-algebra as the Lebesgue measure. The distributions arising from the
above reduction are finitely supported but for the functions considered in this paper, one may
replace finitely supported distributions with distributions that are absolutely continuous with
respect to the Lebesgue measure without changing the results, by replacing each point in the support
with an arbitrarily small ball.

\subsection{Unions of Intervals}

A function $f : \bR \to \zo$ is a \emph{union of $k$ intervals} if there are $k$ intervals
$[a_1,b_1], \dotsc, [a_k,b_k]$, where we allow $a_i = -\infty$ and $b_i = \infty$, such that $f(x) =
1$ iff $x$ is contained in some interval $[a_i,b_i]$. Let $\cI_k$ denote the class of such functions.

The analysis of the LVC dimension of $\cI_k$ is a straightforward variant of the standard analysis of the VC dimension of the class and serves as a good introduction to the high-level structure of the arguments that will be used in later proofs as well.

\begin{proposition}
\label{prop:nc vc intervals}
$\LVC_\bR(\cI_k) = \VC_\bR(\cI_k) = 2k$.
\end{proposition}

\begin{proof}
Let $S \subset \bR$ have size $2k$ and let $\ell : S \to \zo$ be arbitrary. Write $S = \{s_1,
\dotsc, s_{2k}\}$ where $s_1 < \dotsm < s_{2k}$ and partition $S$ into $k$ consecutive pairs
$(s_i,s_{i+1})$ for odd $i$. Then for each pair $(s_i,s_{i+1})$ we can choose a single interval that
contains exactly the points in $s_i,s_{i+1}$ labelled 1 by $\ell$. Therefore $S$ is shattered by $k$
intervals.

On the other hand, let $S \subset \bR$ have size $|S| = 2k+1$, let $s_1 < \dotsm < s_{2k+1}$ be the
points in $S$, and suppose $\ell(i) = 1$ iff $i$ is odd. Then any interval can contain at most 1
point of $S$ labelled $1$, unless it also contains a 0-point. Therefore $S$ is not shattered. So a
set $S$ is shattered iff $|S| \leq 2k$, implying the conclusion.
\end{proof}
Applying \cref{cor:easy bound}, we obtain:
\begin{theorem}
\label{thm:intervals}
For some constant $\epsilon > 0$, $\mteste(\cI_k),\qtolz(\cI_k) =
\Omega\left(\frac{k}{\log k}\right)$.
\end{theorem}

\subsection{Halfspaces}
\label{section:halfspaces}

A halfspace is a function $f : \bR^n \to \pmset$ of the form $f(x) = \sign\left(w_0 + \sum_{i=1}^n
w_ix_i\right)$ where each $w_i \in \bR$. In this subsection, write $\cL_n$ for the class of
halfspaces (or \emph{Linear threshold functions}) with domain $\bR^n$. 

The analysis of the LVC dimension follows immediately from the following well-known shattering properties of halfspaces. (See, e.g.,~\cite{SSBD14}.)

\begin{proposition}
Any set $S \subset \bR^n$ of size $n+1$ in general position can be shattered by $\cL_n$, and any set
$T \subset \bR^n$ of $n$ linearly independent vectors can be shattered by $\cL_n$. No set of size
$n+2$ is shattered by $\cL_n$.
\end{proposition}

Applying \cref{cor:easy bound}, we obtain our lower bound for domain $\bR^n$:
\begin{theorem}
\label{thm:real halfspaces}
For all small enough constant $\epsilon > 0$, the number of 
samples required to test the class $\cL_n$ of halfspaces over
$\bR^n$ satisfies
\[
  \mteste(\cL_n), \qtolz(\cL_n) = 
  \Omega\left(\frac{n}{\log n}\right) \,.
\]
\end{theorem}
\begin{proof}
This holds by \cref{cor:easy bound}, since we may choose any set $S \subset \bR^n$ of size
$|S| \geq 5(n+1)$ in general position, which by the above proposition satisfies
$\LVC_S(\cL_n)=\VC_S(\cL_n)=n+1$.
\end{proof}

\subsection{Intersections of Halfspaces}

Let $\cL^{\cap k}_n$ denote the class of all Boolean-valued functions obtained by taking the intersections of $k$ halfspaces over $\bR^n$. 
Formally, $\cL^{\cap k}_ n$ is the set of functions
\[
  f(x) = h_1(x) \wedge h_2(x) \wedge \dotsm \wedge h_k(x)
\]
where each $h_i$ is a halfspace. It was recently shown by Csik{\'o}s, Mustafa, \&
Kupavskii~\cite{CMK19} that the VC dimension of this classes is
\[
\VC(\cL^{\cap k}_n) = \Theta(nk\log k).
\]
Csik{\'o}s \emph{et al.} remark that it was long assumed (incorrectly) that the VC dimension of
the class was $\Theta(nk)$, which is what one might intuitively expect. We exhibit an infinite
set $S$ on which $\VC_S(\cL^{\cap k}_n) = \LVC_S(\cL^{\cap k}_n) = \Theta(nk)$. 
We do so with an analysis of alternating functions and polynomial threshold functions.

For any $n$, define the mapping $\psi : \bR \to \bR^n$ as follows:
\[
  \psi_n(x) \define \begin{cases}
  (x,x^2,x^3, \dotsc, x^n) &\text{ if $n$ is even} \\
  (0,x,x^2, \dotsc, x^{n-1}) &\text{ if $n$ is odd} \,.
  \end{cases}
\]
Let $\cA_m$ be the set of function $\bR \to \zo$ that alternate at most $m$ times.
\begin{proposition}
\label{prop:ptfs to alternating}
The set $\cP$ of functions $\sign(p(x))$ on $\bR$ where $p$ is a polynomial of degree at most $d$ is
equal to the set $\cA_d$.
\end{proposition}
\begin{proof}
This follows from the fact that number of alternations of the function $\sign(p)$ is exactly the
number of zeroes of $p$, which is at most $d$. On the other hand, any function alternating at most
$d$ times may be represented by $\sign(p)$ where $p$ is a polynomial whose zeroes are exactly the
points where the function alternates.
\end{proof}

\begin{proposition}
For any even $m$ and any $k$, $\cA_m^{\cup k} = \cA_{mk}$.
\end{proposition}
\begin{proof}
It is clear that the union of $k$ $m$-alternating functions will alternate at most $mk$ times, so
$\cA_m^{\cup k} \subseteq \cA_{mk}$, so we must show that $\cA_{mk} \subseteq \cA_m^{\cup k}$. We
will do so by induction on $k$, where the base case $k=1$ is trivial. For $k > 1$, let $f \in
\cA_{mk}$ and let $t_1 < \dotsm < t_{mk}$ be the alternations (i.e.~$f$ is constant on each interval
$(t_i,t_{i+1})$ and $(-\infty,t_1), (t_{mk},\infty)$). There are two
cases: First suppose that the first alternation of $f \in \cA_{mk}$ alternates from 0 to 1; or,
symmetrically, suppose that the last alternation of $f$ alternates from 1 to 0. Then the function
$g$ equal to $f$ on $x \leq t_m$ and 0 on $x > t_m$ is the union of $m/2$ intervals, and $g \in
\cA_m$. Let $f'$ be 0 on $x \leq t_m$ and equal to $f$ on $x > t_m$, so that $f$ is the union of
$f'$ and $g$, and $f' \in \cA_{m(k-1)}$. By induction $f'$ is the union of $k-1$ $m$-alternating
functions, so $f \in \cA_m \cup \cA_m^{\cup (k-1)} = \cA_m^{\cup k}$.

In the second case, the first and last alternations of $f$ alternate from 1 to 0 and 0 to 1,
respectively. Let $g$ take value 1 on $(-\infty,t_1], [t_{mk},\infty)$ as well as on the first
$m/2-1$ intervals $[t_2,t_3], [t_4,t_5], \dotsc, [t_{m-2},t_{m-1}]$, and 0 otherwise. Then $g \in
\cA_m$ and the function $f'=f-g$ is in $\cA_{m(k-1)}$. So by induction $f' \in \cA_m^{\cup (k-1)}$
and $f \in \cA_m \cup \cA_m^{\cup (k-1)} = \cA_m^{\cup k}$.
\end{proof}

\begin{proposition}
\label{prop:alternating vc=nc}
For any even $m$, any $k$, and any set $S \subseteq \bR$ with $|S| > mk, \VC_S(\cA_m^{\cap k}) =
\LVC_S(\cA_m^{\cap k}) = mk$.
\end{proposition}
\begin{proof}
For a class $\cH$, write $\overline \cH$ of the set of functions $f = -g$ where $g \in \cH$
(i.e.~the set of complements of functions in $\cH$). Note that $\overline \cA_m = \cA_m$ since the
complement preserves alternations.  By De Morgan's laws, $(\overline{\cH_n})^{\cap k} =
\overline{\cH^{\cup k}_n}$. Then $\cA_m^{\cap k} = (\overline{\cA_m})^{\cap k} = \overline{\cA_m^{\cup
k}} = \overline{\cA_{mk}} = \cA_{mk}$. The conclusion follows since
$\VC_S(\cA_{mk})=\LVC_S(\cA_{mk})=mk$ by the same argument as for unions of intervals.
\end{proof}

\begin{lemma}
For any $k \geq 1$ and $S \subset \bR$ with $|S| > nk$, if $n$ is even then $\LVC_{\psi_n(S)}(\cL_n^{\cap
k}) = \VC_{\psi_n(S)}(\cL_n^{\cap k}) = nk$ and if $n$ is odd then $\LVC_{\psi_n(S)}(\cL_n^{\cap
k})=\VC_{\psi_n(S)}(\cL_n^{\cap k})=(n-1)k$.
\end{lemma}
\begin{proof}
First suppose that $n$ is even and consider a halfspace $h(y) = \sign(t+\sum_{i=1}^n w_i y_i)$,
where $y = \psi_n(x)$ for some $x \in S$. Then $h(\psi_n(x)) = \sign(t + \sum_{i=1}^n w_i x^i)$,
which is the sign of a degree-$n$ polynomial on $x$. Therefore the set of halfspaces $h$ on the set
$\psi(S)$ is equivalent to the set of degree-$n$ polynomials on $S$, which by
\cref{prop:ptfs to alternating} is equal to the set of $n$-alternating functions, so by
\cref{prop:alternating vc=nc} we have $\LVC_{\psi(S)}(\cL_n^{\cap k}) = \VC_{\psi(S)}(\cL_n^{\cap k}) =
\VC(\cA_n^{\cap k}) = nk$. When $n$ is odd, the same argument shows that $\LVC_{\psi_n(S)}(\cL_n^{\cap
k}) =\VC_{\psi_n(S)}(\cL_n^{\cap k}) = (n-1)k$.
\end{proof}
Applying \cref{cor:easy bound} with a sufficiently large set $S \subset \bR$, we obtain
the theorem:
\begin{theorem}
\label{thm:intersections of halfspaces}
For any $n,k$ and sufficiently small constant $\epsilon > 0$
\[
\mteste(\cL_n^{\cap k}),\qtolz(\cL_n^{\cap k})
  = \Omega\left(\frac{nk}{\log(nk)}\right)\,.
\]
\end{theorem}

\subsection{Decision Trees}

For any parameters $n$ and $k$, let $\cT_{n,k}$ denote the set of functions $f : [0,1]^n \to \zo$ which can be computed by decision trees with at most $k$ nodes,
where each node is of the form ``$x_i < t$?'' for some $t \in \bR$. 
We can bound the LVC dimension of decision trees using the same 
argument as for unions of intervals.

\begin{proposition}
Let $S \subset \bR^n$ be any subset of the line $\{ x \in \bR^n : x_2 = \dotsm = x_n = 0 \}$ with
$|S| > k$. Then $\LVC_S(\cT_{n,k}) = \VC_S(\cT_{n,k}) = k+1$.
\end{proposition}

\begin{proof}
Observe that on any sequence $s_1 < s_2 < \dotsm < s_m$ in $S$, any function $f \in \cT_{n,k}$ can
alternate at most $k$ times, since there are at most $k$ nodes in the decision tree labelled ``$x_1
< t$'' for some values $t$. Therefore $T \subseteq S$ is shattered iff $|T| \leq k+1$.
\end{proof}

Combining this proposition with \cref{cor:easy bound} completes the proof of the lower bound for testing decision trees:

\begin{theorem}
\label{thm:real decision trees}
For any $k$, $n$, and small enough constant $\epsilon > 0$, $\mteste(\cT_{n,k}),\qtolz(\cT_{n,k}) =
\Omega\left(\frac{k}{\log k}\right)$.
\end{theorem}

\section{Classes of Boolean functions}

The techniques used in the last section do not carry over to classes of functions over the Boolean
hypercube. This is because $\pmset^n$ is very far from being in general position---indeed, up to
$2^{n-1}$ points can belong to an affine subspace of dimension $n-1$, by, for example, taking the
subspace obtained by setting the first coordinate to 1.  In this section, we will instead choose the
set $S$ uniformly at random from $\pmset^n$ and show that the properties we need for the reduction
in \cref{lemma:ssd reduction} hold with high probability.

\subsection{Halfspaces}
We first introduce some notation and a theorem that will be used also for PTFs in the next
subsection.  For a vector $a \in \zo^n$ and $x \in \bR^n$ we will write $x^a = \prod_{i=1}^n
x_i^{a(i)}$.  Write $|a|=\sum_i a(i)$.  Let $\psi_k : \bR^n \to \bR^{n \choose \leq k}$ be defined
as follows:
\[
  \psi_k(x) = (x^a)_{a \in \zo^n : |a| \leq k} \,.
\]
We will use the following theorem of Abbe, Shpilka, \& Wigderson \cite{ASW15}:
\begin{theorem}[\cite{ASW15}]
\label{thm:linear independence}
Let $n,k,m$ be positive integers such that
\[
  m < { n - \log{n \choose \leq k} - t \choose \leq k } \,.
\]
Then for independent, uniformly random vectors $x_1, \dotsc, x_m \sim \pmset^n$, the vectors
$\psi_k(x_1), \dotsc, \psi_k(x_m) \in \pmset^{n \choose \leq k}$ are linearly independent with
probability at least $1-2^{-t}$.
\end{theorem}

Let $\cL_{n}$ denote the set of halfspaces (or linear threshold functions) over $\pmset^n$.

\begin{theorem}
\label{thm:boolean halfspaces}
For every $n$ and all sufficiently small constant $\epsilon > 0$,
\[
    \mteste(\cL_n), \qtolz(\cL_n) = \Omega\left(\frac{n}{\log n}\right) \,.
\]
\end{theorem}
\begin{proof}
Set $m = 5(n+1), \alpha = 1/11, \beta = 4/5$.
We will repeat the reduction from $\SSD(m,\alpha,\beta)$ to testing $\cL_n$ as in \cref{lemma:ssd
reduction} and \cref{thm:main} with the fixed set $S$ replaced by a random set $S$ of size
$m$ drawn from $\pmset^n$. First suppose that the input distribution $\cD$ over $[m]$ has support
size at most $\alpha m < n/2$. Then $T \define \supp(\phi \cD)$ is a uniformly random subset of
$\pmset^n$ of size at most $n/2$, so since $|T| \leq n/2 < n-\log(1+n)-C$ for any constant $C$, by
\cref{thm:linear independence} (with $k=1$), the points in $T$ are linearly independent with
probability at least $9/10$. In this case, $T$ is shattered by $\cL_n$, so the remainder of the proof
goes through as in \cref{lemma:ssd reduction}. When $\cD$ has support size at least $\beta m =
4(n+1)$, the proof goes through as in \cref{lemma:ssd reduction} and \cref{cor:easy
bound} with the constant $K=4$, and we obtain the lower bound.
\end{proof}

\subsection{Polynomial Threshold Functions}

Let $\cP_{n,k}$ denote the class of polynomial threshold functions with degree $k$ over $\pmset^n$.
The above mapping $\psi_k : \pmset^n \to \pmset^d$ with $d = \binom{n}{\leq k}$ establishes an
equivalence between PTFs and halfspaces in a higher dimension:
\begin{lemma}
\label{lemma:ptf to halfspace}
Write $d = {n \choose \leq k}$.  A set $S \subseteq \bR^n$ is shattered by $\cP_{n,k}$ if and only
if $\psi_k(S)$ is shattered by $\cL_d$.
\end{lemma}
\begin{proof}
We shall index the coordinates of $\pmset^d$ with vectors $a \in \zo^n$ satisfying $|a| \leq k$.
Let $\ell : S \to \pmset$ be any labelling of $S$. Note that $\psi_k$ is a bijection (which can be
seen just from the vectors $a$ with $|a|=1$. If there is a degree-$k$ polynomial $p(x) =
\sum_{a \in \zo^n, |a|\leq k} w_a x^a$ such that $\sign(p(x))=\ell(x)$ for every $x \in S$, then for
every $x \in S$ we have
\begin{align*}
  \ell(x) = \sign(p(x)) = \sign\left(w_0 + \sum_{a \in \zo^n, |a|\leq k} w_ax^a\right)
          = \sign\left(w_0 + \sum_{a \in \zo^n, |a|\leq k} w_a \psi_k(x)_a \right) \,.
\end{align*}
Observe that the function on the right is an LTF in $\cL_d$, so there is an LTF consistent with the
labelling $\ell \circ \psi_k^{-1}$ on $\psi_k(S)$. So, if $S$ is shattered by $\cP_{n,k}$ then
$\psi_k(S)$ is shattered by $\cL_d$, because $\psi_k$ acts also as a bijection between labellings of
$S$ and $\psi_k(S)$. On the other hand, the same equation shows that for any labelling $\ell :
\psi_k(S) \to \pmset$, if there is an LTF $f : \bR^d$ such that $f(\psi_k(x))=\ell(\psi_k(x))$ for
each $x \in \psi_k(S)$ then there is a PTF $g : \bR^n \to \pmset$ such that $g(x) = f(\psi(x)) =
\ell(\psi(x))$ for each $x \in S$. Therefore $S$ is shattered by $\cP_k$ iff $\psi_k(S)$ is
shattered by $\cL_d$.
\end{proof}

\begin{theorem}
\label{thm:boolean ptfs}
Write $\cP^\pm_{n,k}$ for the set of degree-$k$ PTFs with domain $\pmset^n$.  There exists some
constant $C'$ such that for all $k < n/C'$ and for sufficiently small constant $\epsilon > 0$,
\[
  \mteste(\cP^\pm_{n,k}), \qtolz(\cP^\pm_{n,k})
  = \Omega\left(\frac{{n -\log{n \choose k} - O(1) \choose \leq k}^2}{{n \choose \leq k}\log{n
\choose \leq k}}\right)
  = \Omega\left(\frac{(n/4ek)^k}{k\log(n/k)}\right) \,.
\]
\end{theorem}
\begin{proof}
Let $d = {n \choose \leq k}$ and set $m = 5d$. Let $\beta = 4/5$, $t = \log(10)$, and
\[
  \alpha \define \frac{1}{5} {n \choose \leq k}^{-1} {n - \log{n \choose \leq k} - t \choose \leq k}
\]
As was the case with halfspaces, we let $S$ be a uniformly random set of $m$ points drawn from
$\pmset^n$, let $\phi : [m] \to S$ be a random mapping obtained by assigning a uniform and
independently random $x \in S$ to each $i \in [m]$, and complete the reduction from
$\SSD(m,\alpha,\beta)$ to testing $\cP_{n,k}$ as in \cref{lemma:ssd reduction} and \cref{thm:main},
which we verify below.

We must first verify that $\alpha \geq C\frac{\sqrt{\log m}}{m^{1/4}}$, where $C$ is the constant
in \cref{thm:wy easy application}, for which it suffices to
prove that $\alpha \geq \hat C\frac{\sqrt{\log d}}{d^{1/4}}$ for a slightly larger $\hat C > C$,
since $m=5d$.  For an appropriately large choice of constant $C'$, and sufficiently large $n > 2t$,
\begin{align*}
\log{n \choose \leq k}+t
&\leq \log{n \choose \leq n/C'} +t
\leq \log\left(\left(\frac{en}{n/C'}\right)^{n/C'}\right) +t
\leq \log\left((C')^{n/C'}\right)+t \\
&= \frac{n}{C'}\log(eC') + t \leq n/2 \,,
\end{align*}
so
\[
  \alpha \geq \frac{1}{5} {n \choose \leq k}^{-1} {n/2 \choose \leq k}
  \geq \frac{1}{5} \left(\frac{n}{2k}\right)^k \left(\frac{k}{en}\right)^k
  = \left(\frac{1}{2e}\right)^k \,.
\]
For any constant $\eta > 0$, we may assume $C' > (\hat C2e)^{\frac{1}{1/4-\eta}}$, so that, using
$\frac{k}{n} \leq \frac{1}{C'} \leq \frac{1}{(C2e)^{\frac{1}{1/4-\eta}}}$, we get
\[
  \hat C\frac{\sqrt{\log d}}{d^{1/4}} \leq C\frac{1}{d^{1/4-\eta}}
  \leq \hat C\left(\frac{k}{n}\right)^{k(1/4-\eta)}
  \leq \hat C\left(\frac{1}{(\hat C2e)^{\frac{1}{1/4-\eta}}}\right)^{k(1/4-\eta)}
  \leq \frac{1}{5} \left(\frac{1}{2e}\right)^k \leq \alpha \,.
\]
Now we verify correctness.  Suppose that the input distribution $\cD$ over $[m]$ has support size at
most $\alpha m$ and let $T \define \supp(\phi\cD)$.  $T$ is a (multi)set of at most
\[
\alpha m = d {n \choose \leq k}^{-1} {n - \log{n \choose \leq k} - t \choose \leq k}
         = {n - \log{n \choose \leq k} - t \choose \leq k}
\]
uniformly random points from $\pmset^n$, so by \cref{thm:linear independence} the probability
that the points $\psi_k(T)$ are linearly independent is at least $9/10$. In that case, $\psi_k(T)$
is shattered by the halfspaces $\cH_d$ over $\pmset^d$ so by \cref{lemma:ptf to halfspace}, $T$
is shattered by $\cP_{n,k}$. Therefore, as in \cref{lemma:ssd reduction}, the tester for
$\cP_{n,k}$ will output 1 with probability at least $5/6$, so the distribution $\cD$ is accepted
with probability at least $2/3$.

Now suppose that the input distribution $\cD$ over $[m]$ has support size at least $\beta m =4d$,
and let $T = \supp(\phi\cD)$. Since $\phi$ is a random mapping (with replacement), we must first
show that, with high probability, $|T| \geq Kd$ for the constant $K > 3.04$ in
\cref{lemma:random labelling}. Since $k \leq n / C'$ for a sufficiently large constant $C'$, we have
$4d = 4{n \choose \leq k} \leq 4(eC')^{n/C'} \leq 2^{cn}$ for constant $c < 1/3$. Therefore the
probability that a random point $x$ in $T$ is unique is at least $1 - \frac{4d}{2^n} \geq 1 -
2^{(c-1)n}$. By the union bound, the probability that any point fails to be unique is at most
$4d2^{(c-1)n} = 4{n \choose \leq k}2^{(c-1)n} \leq 2^{(2c-1)n} < 2^{-n/3}$. When this occurs, the
support of $\phi\cD$ has size at least $4d$ so, as in \cref{thm:main}, we may apply
\cref{lemma:random labelling} to conclude that a random labelling $f : \pmset^n \to \pmset$
satisfies $\dist_{\phi\cD}(f,\cP_{n,k}) \geq \epsilon$ with probability at least $9/10$, for some
small enough constant $\epsilon > 0$. Then the tester for $\cP_{n,k}$ will output 0 with probability
at least $5/6$, so the distribution $\cD$ is rejected with probability at least $2/3$.

We obtain a lower bound of $\Omega\left(\frac{d}{\log d}\log^2\frac{1}{1-\alpha}\right)$, since
$1-\beta \geq \alpha$.
Using the inequality
$\log^2 \frac{1}{1-x} \geq \log^2\frac{1}{e^{-x}} = \log^2(e^{x}) = \Omega(x^2)$,
we get
\[
\frac{d}{\log d}\log^2\frac{1}{1-\alpha}
= \Omega\left(\frac{d}{\log d}\alpha^2\right)
    = \Omega\left(\frac{{n-\log(d)-t \choose \leq k}^2}{d\log d}\right) \,.
\]
To obtain the simplified bound, use $n-\log(d)-t \leq n/2$ from above, and ${n/2 \choose \leq k}
\geq (n/2k)^k$ to get
\[
\Omega\left(\frac{(n/2k)^{2k}}{d\log d}\right) 
= \Omega\left(\frac{(n/2k)^{2k}}{(en/k)^k k\log(en/k)}\right) 
= \Omega\left(\frac{(n/4ek)^{k}}{ k\log(n/k)}\right) \,. \qedhere
\]
\end{proof}

\subsection{Decision Trees}

Let $\cB_{n,k}$ be the set of functions $f : \zo^n \to \zo$
defined by decision trees with $k$ nodes of the form ``$x_i=1$?''.
When $k \gg \log n$, fairly tight bounds on the VC dimension of
$\cB_{n,k}$ are known.

\begin{lemma}[Mansour \cite{Man97}]
\label{lemma:mansour}
$\VC(\cB_{n,k})$ is between $\Omega(k)$ and $O(k\log n)$.
\end{lemma}

A lower bound on the LVC dimension of $\cB_{n,k}$ is also easily established.

\begin{proposition}
\label{prop:lvc boolean decision tree}
Every subset $T \subseteq \zo^n$ of size at most $k$ is shattered by $\cB_{n,k}$.
\end{proposition}

\begin{proof}
We prove by induction on $k$ that any set $T \subseteq S$ of size $k$ is shattered by a decision
tree with at most $k$ leaves.  Clearly when $k=1$, for any subset $T \subseteq S$ of size $|T|=1$,
decision trees with 0 nodes and 1 leaf shatter $T$. For $k > 1$, there exists a coordinate $i \in
[n]$ such that $T_0 \define \{ x \in T : x_i = 0\} \neq \emptyset$ and $T_1 \define \{ x \in T : x_i
= 1\} \neq \emptyset$. Now $T_0$ is a subset of size $k-|T_1| < k$ so by induction it is shattered
by subtrees with at most $k-|T_1|$ leaves, while $T_1$ is shattered by subtrees with at most $|T_1|$
leaves. Therefore $T$ is shattered by a tree with at most $k$ leaves. Since the number of nodes is
at most the number of leaves, $T$ is shattered by $\cB_{n,k}$.
\end{proof}

We are now ready to bound the sample and tolerant-query complexities for testing decision trees.

\begin{theorem}
\label{thm:boolean decision trees}
For any $k$, $n \ge \log k + \log \log k + \Omega(1)$, and sufficiently small constant $\epsilon > 0$,
\[
\mteste(\cB_{n,k}), \qtolz(\cB_{n,k}) = \Omega\left(\frac{k}{\log k \cdot \log\log k}\right) \,.
\]
\end{theorem}
\begin{proof}
Let $S \subset \zo^n$ be a subcube with dimension $m = \log(6C)+\log k+\log\log\log k$ and let $d
= \VC_S(\cB_{n,k})$. Then by \cref{lemma:mansour}, for some constant $C$ and sufficiently large $k$,
\[
d \leq Ck\log(m) = Ck\log\log(6Ck\log\log k) \leq Ck\log\log(k^2)
= Ck(\log\log k + 1) \,,
\]
so that
\[
(1-\delta)|S| = (1-\delta)2^m = 6Ck\log\log k - k
  = 5Ck\log\log k + Ck(\log\log k-1/C)
  \geq 5Ck(\log\log k + 1) \geq 5d \,.
\]
By \cref{prop:lvc boolean decision tree}, $\LVC_S(\cB_{n,k}) \geq k$, so for $\delta =
\frac{1}{6C\log\log k}$,
\[
  \LVC_S(\cB_{n,k}) \geq k = \delta 6Ck\log\log k = \delta |S| \,,
\]
therefore the conditions for \cref{thm:main} are satisfied. We obtain a lower bound of
\[
  \Omega\left(\frac{k\log\log k}{\log k} \log^2\frac{1}{1-\frac{1}{\log\log k}}\right) \,.
\]
Using the inequality
$\log^2 \frac{1}{1-1/x} \geq \log^2\frac{1}{e^{-1/x}} = \log^2(e^{1/x}) = \Omega(1/x^2)$,
we get
\begin{align*}
  \Omega\left(\frac{k\log\log k}{\log k} \log^2\frac{1}{1-\frac{1}{\log\log k}}\right)
  &= \Omega\left(\frac{k\log\log k}{\log k} \log^2\frac{1}{1-\frac{1}{\log\log k}}\right) \\
  &= \Omega\left(\frac{k\log\log k}{(\log k)(\log\log(k))^2}\right) \\
  &= \Omega\left(\frac{k}{\log k \cdot \log\log k}\right) \,. \qedhere
\end{align*}
\end{proof}

\section{Maximum Classes and Analytic Dudley Classes}
\label{section:dudley}

A number of sample complexity lower bounds for testing natural classes of functions can be obtained by considering maximum and analytic Dudley classes, as we describe in this section. 

\subsection{LVC and the Sauer-Shelah-Perles Lemma}
\label{section:sauer lemma}

Recall the Sauer-Shelah-Perles lemma and the associated definitions:

Let $\cH$ be a set of functions $\cX \to \zo$ and let $S \subseteq \cX$.
The \emph{shattering number} is
\[
  \sh(\cH,S) \define |\{ T \subseteq S \;|\; T \text{ is shattered by } \cH \}| \,,
\]
and the \emph{growth function} is
\[
  \Phi(\cH,S) \define |\{ \ell : S \to \zo \;|\; \exists h \in \cH \; \forall x \in S,
\ell(x)=h(x)\}| \,.
\]

\textbf{Sauer-Shelah-Perles lemma. }\emph{%
Let $\cH$ be a class of functions $\cX \to \zo$ and let $S \subseteq \cX$ with $\VC_S(\cH) = d$.
Then $\Phi(\cH,S) \leq \sh(\cH,S) \leq \sum_{i=0}^d { |S| \choose i }$.}

Much research has studied the cases where this inequality is tight in various ways: A class is
called \emph{maximum} on $S$ (\cite{GW94,FW95,KW07,Joh14,AMY16,MW16,CCMW19}) if the sequence
of inequalities is tight, i.e.~$\cH$ is maximum on $S$ if
\[
  \Phi(\cH,S) = \sh(\cH,S) = \sum_{i=0}^d { |S| \choose i } \,.
\]
A class is called \emph{shatter-extremal} on $S$ (see e.g.~\cite{Mor12,MW16,CCMW19}) if the first
inequality is tight, i.e.
\[
  \Phi(\cH,S) = \sh(\cH,S) \,.
\]
We are not aware of any studies of the case where the second inequality $\sh(\cH,S) \leq
\sum_{i=0}^d \binom{|S|}{i}$ is tight; our requirement $\LVC_S(\cH) = \VC_S(\cH)$ fills in the gap:
\begin{proposition}
A set $\cH$ of functions $\cX \to \zo$ satisfies $\LVC_S(\cH) = \VC_S(\cH)$ on a set $S \subseteq
\cX$ if and only if $\sh(\cH,S) = \sum_{i=0}^d { |S| \choose i }$, for $d = \VC_S(\cH)$.
\end{proposition}
\begin{proof}
This follows from the fact that $\sum_{i=0}^d { |S| \choose i}$ is exactly the number of sets of
size at most $d$; if the equality holds, all such sets are shattered, so $\LVC_S(\cH) = d$. On the
other hand if $\LVC_S(\cH) = d$ then all sets of size at most $d$ are shattered, so the equality
holds.
\end{proof}
We can therefore conclude:
\begin{proposition}
\label{prop:maximum condition}
A set $\cH$ of functions $\cX \to \zo$ is maximum on $S \subseteq \cX$ if and only if it is both
shatter-extremal on $S$ and $\LVC_S(\cH) = \VC_S(\cH)$.
\end{proposition}
Then we easily obtain lower bounds for maximum classes using \cref{cor:easy bound}.
\begin{theorem}
\label{thm:maximum}
Let $\cH$ be a set of functions $\cX \to \zo$. Suppose there is $S \subseteq \cX$ such that $\cH$ is
maximum on $S$ and $d \define \VC_S(\cH)$ satisfies $|S| \geq 5d$. Then for sufficiently small
constant $\epsilon > 0$,
\[
  \mteste(\cH), \qtolz(\cH) = \Omega\left(\frac{d}{\log d}\right) \,.
\]
\end{theorem}
Examples of maximum classes include the set of functions $f : [n] \to \zo$ with at most $n/5$
1-valued points \cite{MW16} (studied in \cref{section:optimality}), unions of $k$
intervals \cite{Floyd89}, and positive halfspaces (halfspaces with normal vectors $w \in \bR^n$
satisfying $x_i \geq 0$) \cite{FW95}. Another standard example is the set of sign vectors arising
from an arrangement of hyperplanes:
\begin{example}[\cite{GW94}]
Let $H$ be a set of $n > d$ hyperplanes in $\bR^d$ and write $H = \{h_1, \dotsc, h_n\}$ where each
$h_i : \bR^d \to \pmset$ is of the form $h_i(x) = \sign(t + \sum_{j=1}^d w_j x_j)$ for some $t,w_j
\in \bR$. Assume that the hyperplanes are in general position. Let $\cH$ be the set of functions
$f_x : [n] \to \pmset$ obtained by choosing $x \in \bR^d$ obtained by setting $f_x(i) = h_i(x)$.
Then $\VC_{[n]}(\cH) = d$ and $\cH$ is maximum on $[n]$, as proved by Gartner \& Welzl \cite{GW94}.
Therefore, for any such set $\cH$ where $n \geq 5d$ we obtain via \cref{thm:maximum} that
$\mteste(\cH),\qtolz(\cH) = \Omega(d / \log d)$.
\end{example}

\subsection{Analytic Dudley Classes}
Some examples of maximum classes and classes with $\LVC_S(\cH)=\VC_S(\cH)$ that are arguably more
pertinent to property testing can be obtained from a family of classes called \emph{Dudley classes}
\cite{BL98}.
\begin{definition}[Dudley Class]
A class $\cH$ of functions $\cX \to \pmset$ is a \emph{Dudley class} if there exists a set $\cF$ of
functions $X \to \bR$ and a function $h : X \to \bR$ such that:
\begin{itemize}
\item $\cF$ is a vector space, i.e.~$\forall f,g \in \cF, \lambda \in \bR$, $f+g \in \cF$ and
$\lambda f \in \cF$;
\item Every $g \in \cH$ can be written as $g(x) = \sign(f(x)+h(x))$.
\end{itemize}
We will refer to $\cF$ as the vector space of $\cH$ and $h$ as the threshold of $\cH$.
\end{definition}
The VC dimension of Dudley classes is equal to the dimension of the vector space $\cF$:
\begin{theorem}[\cite{WD81} Theorem 3.1]
\label{thm:dudley}
Let $\cH$ be any Dudley class with vector space $\cF$. Then $\VC(\cH) = \dim(\cF)$.
\end{theorem}
This theorem implies that $\LVC_S(\cH)=\VC_S(\cH)$ on a set $S \subseteq \cX$ if and only
if the dimension of the vector space remains the same when restricted to any subset of $S$:
\begin{corollary}
Let $\cH$ be a Dudley class of functions $\cX \to \pmset$ with vector space $\cF$ of functions $\cX
\to \bR$ and threshold $h$.  Then for any set $S \subseteq \cX$, $\VC_S(\cH) = \LVC_S(\cH)$ if and
only if the vector space $\cF$ restricted to any $T \subseteq S$ of size $|T|=d=\VC_S(\cH)$ has
dimension $d$.
\end{corollary}
\begin{proof}
This follows from the above theorem, since for any $T \subseteq S$ of size $|T|=d$ on which $\cF$
has dimension $d$, $\VC_T(\cH) = d$, so $T$ is shattered.
\end{proof}
A useful condition on Dudley classes that guarantees the above condition was described by Johnson
\cite{Joh14}. Recall that a function $f : \bR^n \to \bR$ is \emph{analytic} if it is infinitely
differentiable and for every $x$ in the domain, there is an open set $U \ni x$ such that $f$ is
equal to its Taylor series expansion on $U$. We will call a Dudley class \emph{analytic} if its
threshold $h$ and each $f$ in the basis of $\cF$ is analytic. Johnson proves the following
(rewritten in our terminology):
\begin{theorem}[\cite{Joh14}]
Let $\cH$ be any analytic Dudley class on domain $[0,1]^n$ with $\VC(\cH)=d$. Then for any
$N > n$ there exists a set $S \subset [0,1]^n$ of size $|S|=N$ such that $\cH$ is maximum on $S$
with $\VC_S(\cH) = d$.
\end{theorem}
Then by taking $N \geq 5d$ in the above theorem and applying \cref{thm:maximum}, we obtain:
\begin{corollary}
\label{cor:analytic dudley classes}
Let $\cH$ be any analytic Dudley class and suppose $\VC(\cH) = d$. Then for some constant
$\epsilon > 0$,
\[
  \mteste(\cH), \qtolz(\cH) = \Omega\left(\frac{d}{\log d}\right) \,.
\]
\end{corollary}
Examples of analytic Dudley classes include halfspaces (for which we have already proved the lower
bound) and PTFs. Other examples due to \cite{Joh14} are balls in $\bR^n$ and trigonometric
polynomial threshold functions in $\bR^d$:
\begin{theorem}
\label{thm:analytic dudley}
For sufficiently small constant $\epsilon > 0$, the following classes $\cH$ satisfy the given lower
bounds for both $\mteste(\cH)$ and $\qtolz(\cH)$:
\begin{enumerate}
\item Degree-$k$ PTFs on domain $\bR^n$ satisfy the lower bound
  $\Omega\left(\frac{{n+k \choose k}}{\log {n+k \choose k}}\right)$.
\item Balls in $\bR^n$, i.e.~functions $f : \bR^n \to \pmset$ of the form $f(x)= \sign(t -
\|x-z\|_2)$, satisfy the lower bound $\Omega\left(\frac{n}{\log n}\right)$.
\item Signs of trigonometric polynomials, i.e. functions $\bR^2 \to \pmset$ of the form:
\[
  f(x,y) = \sign\left(t + \sum_{k=1}^d a_k \cos(kx) + \sum_{k=1}^d b_k \sin(kx) - y\right) \,,
\]
which satisfy the lower bound $\Omega\left(\frac{d}{\log d}\right)$.
\end{enumerate}
\end{theorem}

\section{Other Models of Testing}

In this section, we complete the proofs of the two additional results that we obtain using the same framework as the results above: the lower bounds for testing clusterability and for testing feasibility of LP-type problems.

\subsection{Testing Clusterability}
\label{section:clustering}

For a point $x \in \bR^n$ and radius $r > 0$, define $B_r(x) = \{ y \in \bR^n : \|x-y\|_2 \leq r
\}$.  Alon, Dar, Parnas, \& Ron \cite{ADPR03} introduced the problem of testing clusterability with
radius cost:
\begin{definition}[Radius Clustering]
Say that a probability distribution $\cD$ over $\bR^n$ is \emph{$k$-clusterable} if there exist $k$
centers $c_1, \dotsc, c_k \in \bR^n$ such that $\supp(\cD) \subseteq \cup_{i=1}^k B_1(c_i)$. An
$\epsilon$-tester for $k$-clusterability is a randomized algorithm $A$ that is given sample access
to $\cD$ and must satisfy the following:
\begin{enumerate}
\item If $\cD$ is $k$-clusterable then $\Pr{A(\cD) = 1} \geq 2/3$; and,
\item If $\cD$ is $\epsilon$-far from being $k$-clusterable in total variation distance, then
$\Pr{A(\cD) = 0} \geq 2/3$.
\end{enumerate}
\end{definition}

Alon \emph{et al.}~\cite{ADPR03} prove an upper bound of $O\left(\frac{nk\log(nk)}{\epsilon}\right)$
samples for one-sided testing of $k$-clusterability when the distribution is uniform over an unknown
set of points. Their proof is by VC dimension arguments. The following theorem updates the upper
bound of \cite{ADPR03} using modern VC dimension results; it follows from the same $\epsilon$-net argument found in \cref{lemma:one-sided upper bound} (see also \cite{H14}), and the fact that the VC dimension of unions of $k$ balls is at most $O(nk \log k)$ \cite{CMK19}.
\begin{theorem}[Improved version of \cite{ADPR03}]
\label{thm:clustering upper bound}
There is a one-sided, distribution-free $\epsilon$-tester for $k$-clusterability in $\bR^n$ with
sample complexity $O\left(\frac{nk\log k}{\epsilon}\log\frac{1}{\epsilon}\right)$.
\end{theorem}

In this section, we prove a nearly-optimal $\Omega(nk/\log(nk))$ lower bound on distribution-free
testers for $k$-clusterability with two-sided error.

Let $S^n_r = \{ x \in \bR^n : \|x\|_2 = r \}$ be the points on the hypersphere of radius $r$.
\begin{proposition}
\label{prop:<n}
For every $\delta > 0$ there is $\eta > 0$ such that a uniformly random set of $n$ points $P$ drawn
from $S^n_{1+\eta}$ is contained within some ball $B_1(x)$ with probability at least $1-\delta$.
\end{proposition}
\begin{proof}
Unless all $n$ points in $P$ lie on a hyperplane through the origin (which occurs with probability
0), there is a hyperplane through the origin such that all points in $P$ lie on one side. Consider
the distribution of $P$ conditional on this event, and without loss of generality assume that the
hyperplane is $\{ x : x_1=0\}$ so that all points $x \in P$ satisfy $x_1 > 0$. Let $\eta > 0$ and
consider the ball $B$ of radius 1 centered at $z=(\sqrt{(1+\eta)^2-1},0,\dotsc,0)$. Let $x \in
S^n_{1+\eta}$ satisfy $x_1 \geq z_1 = \sqrt{(1+\eta)^2-1}=\sqrt{\eta(2-\eta)}$. Then since $\|x\|_2^2 = (1+\eta)^2$,
\begin{align*}
  \|x-z\|_2^2
    &= (x_1-z_1)^2 + \sum_{i=2}^n x_i^2 = (x_1-z_1)^2 + (1+\eta)^2 - x_1^2 \\
    &= z_1^2 - 2x_1z_1 + (1+\eta)^2 \leq (1+\eta)^2 - z_1^2 = 1 \,,
\end{align*}
so all points $x$ with $x_1 \geq z_1$ are contained within the ball $B$. Conditioned on $x_1 > 0$,
the probability that $x_1^2 \geq \eta(2-\eta)$ is at least the probability that $y_1^2 \geq
\eta(2-\eta)$ for $y$ drawn uniformly randomly from $S^n_1$. This probability goes to 1 as $\eta \to
0$, so the probability that $x_1^2 \geq \eta(2-\eta)$ also approaches 1 as $\eta \to 0$. The
conclusion follows.
\end{proof}

\begin{proposition}
\label{prop:>n}
For every constant $\delta,\eta > 0$, there is a constant $\epsilon_0 > 0$ such that, for all
$\epsilon < \epsilon_0$ and for a uniformly random set $P$ of $m=2n$ points drawn from
$S^n_{1+\eta}$, with probability at least $1-e^{-\delta n}$, no subset $T \subset P$ of size
$(1-\epsilon)m$ is contained within a ball of radius 1.
\end{proposition}
\begin{proof}
Let $t=(1-\epsilon)m > n$ and let $T \subset P$ have size $|T|=t$. If the points $T$ are contained
within a ball of radius 1 then they are contained within a centered halfspace, because the
intersection of the ball with $S^n_{1+\eta}$ is equal to the intersection of some halfspace with
$S^n_{1+\eta}$. The probability that $t$ uniformly random points on the surface of the sphere
lie within some hemisphere is $2^{1-t} \sum_{k=0}^{n-1}{t-1 \choose k}$ \cite{Wen62}. There are at
most ${m \choose t}$ subsets of size $t$, so the probability that any of these subsets lie within a
hemisphere is at most
\begin{align*}
  {m \choose m-t} 2^{1-t} \left(\frac{et}{n}\right)^{n}
&\leq 2^{1-(1-\epsilon) m}\left(\frac{e}{\epsilon}\right)^{\epsilon m} \left(\frac{et}{n}\right)^{n} \\
&= 2^{1+\epsilon m \log(e/\epsilon) - (1-\epsilon)m + n\log\left(\frac{e(1-\epsilon)m}{n}\right)} \\
&= 2^{1+\epsilon 2n \log(e/\epsilon) - (1-\epsilon)2n + n\log(e(1-\epsilon)2)} \\
&\leq 2^{1-2n(1-\epsilon \log(4e^2/\epsilon))} \,.
\end{align*}
The conclusion holds since $\epsilon\log(4e^2/\epsilon) \to 0$ as $\epsilon \to 0$.
\end{proof}

\begin{proposition}[Balls and bins]
  \label{prop:load bounds}
  Fix $C>0$, $0 < \delta \le 1$, and let $n, k$ be positive integers with $k \le \frac{1}{10} e^{\delta^2 Cn / 3}$.
  Then if $Cnk$ balls are deposited into $k$ bins uniformly at random, the following hold:
  \begin{enumerate}
  \item With probability at least $9/10$, every bin receives at most $(1+\delta)Cn$ balls;
  \item With probability at least $9/10$, every bin receives at least $(1-\delta)Cn$ balls.
  \end{enumerate}
\end{proposition}
\begin{proof}
Let $X_{ij}$ be the indicator variable for the event that the $i$-th ball goes into the $j$-th
bin, and let the random variable $L_j = \sum_{i=1}^{Cnk} X_{ij}$ denote the final load on the $j$-th
bin. Note that $\Ex{L_j} = Cn$. By the multiplicative Chernoff bound, we have:
  \begin{enumerate}
  \item $\Pr{L_j \ge (1+\delta)Cn} \le e^{-\delta^2Cn/3}$; and
  \item $\Pr{L_j \le (1-\delta)Cn} \le e^{-\delta^2Cn/3}$.
  \end{enumerate}
  In both cases, by the union bound, the probability that the respective event occurs for any
  $L_j$ ($1 \le j \le k$) is at most $k \cdot e^{-\delta^2Cn/3} \le 1/10$, as desired.
\end{proof}

\begin{lemma}
\label{lemma:clustering}
For $k < \frac{1}{10} e^{n/6}$, let $A_1, \dotsc, A_k$ be spheres in $\bR^n$ of radius $1+\eta$ for
sufficiently small $\eta > 0$, such that the minimum distance between any two spheres is 3. Define
the following distribution $\cS$ over $\bigcup_{i=1}^n A_i$: Draw $i \in [k]$ uniformly at random
and then draw $x \sim A_i$ uniformly at random. Then:
\begin{enumerate}
\item If $S$ is a set of $m \le nk/2$ independent points drawn from $\cS$, then with probability at
least $9/10$, there are $k$ balls of radius 1 whose union contains $S$;
\item If $S$ is a set of $4nk \le m \le 8nk$ independent points drawn from $\cS$ and $\epsilon > 0$
is a sufficiently small constant, then with probability at least $81/100$, no union of $k$ balls of
radius 1 contains more than $(1-\epsilon) m$ points of $S$.
\end{enumerate}
\end{lemma}
\begin{proof}
First suppose that $m \le nk/2$. If each sphere $A_i$ receives at most $n$ sample points then by
\cref{prop:<n}, setting $\delta,\eta > 0$ arbitrarily small in the statement of that proposition,
for each sphere $A_i$ there is a ball $B_i$ of radius 1 containing all points $S \cap A_i$ with
probability arbitrarily close to 1, so there are $k$ balls containing all points of $S$.
\cref{prop:load bounds} (with $C=1/2$ and $\delta=1$) shows that the maximum load of any sphere is
at most $n$ with probability at least $9/10$, so the first conclusion holds.

Now suppose that $4nk \le m \le 8nk$. Note that no ball of radius 1 can contain points from more than
1 sphere $A_i$. \cref{prop:load bounds} (with $C=4$ and $\delta=1/2$) shows that
the minimum load of any sphere is at least $2n$ with probability at least $9/10$. Assume that this
occurs for the rest of this argument.

Let $S_i = S \cap A_i$ for $i = 1, \dotsc, k$, and say that $S_i$ is \emph{difficult} if no ball
of radius 1 contains at least $(1-\epsilon')|S_i|$ points in $S_i$, for constant $\epsilon'$ to be
defined. Since $|S_i| \ge 2n$,
\cref{prop:>n} gives that $\Pr{S_i \text{ is difficult}} \ge 1 - e^{-\delta n}$.
Setting $\delta=1/6$ and by the union bound, the probability that every $S_i$ is difficult is at least
$1 - k \cdot e^{-\delta n} \ge 1 - \frac{1}{10} e^{n/6} e^{-\delta n} = 9/10$. Fix $\epsilon'$ corresponding
to $\delta=1/6$ in \cref{prop:>n}.

Assume that every $S_i$ is difficult, and consider any set of $k$ balls $B_1, \dotsc, B_k$.
Denote their union by $B = \bigcup_i B_i$. Then for each $S_i$,
we have that $|B \cap S_i| \ge (1-\epsilon')|S_i|$ only if at least two balls $B_{j_1}, B_{j_2}$
intersect $S_i$. Thus, this can only happen for at most $k/2$ such $S_i$'s. Assume without loss of
generality that $S_1, \dotsc, S_\ell$ have at least $(1-\epsilon')$-fraction of their points covered
by $B$, so that $\ell \le k/2$. It follows that 
\[
|S \setminus B| \geq \sum_{i=\ell+1}^k \epsilon'|S_i| \ge \frac k2 \cdot \epsilon' \cdot 2n \ge \frac{\epsilon' m}{8} \,.
\]
Which satisfies the second claim for $\epsilon = \epsilon'/8$, and this happens with probability
at least $9/10 \cdot 9/10 = 81/100$ over the choice of $S$.
\end{proof}

\begin{theorem}[Restatement of \cref{thm:clustering intro}]
\label{thm:clustering lower bound}
For sufficiently small constant $\epsilon > 0$, any $\epsilon$-tester for $k$-clusterability in
$\bR^n$ requires at least $\Omega\left(\frac{nk}{\log(nk)}\right)$ samples.
\end{theorem}
\begin{proof}
Let $N = 8nk$ and let $\alpha = 1/16, \beta = 1/2$. We will prove a reduction from support-size
distinction to $k$-clusterability; we may assume that the tester for $k$-clusterability has success
probability at least $5/6$ due to standard boosting techniques. For an input distribution $\cD$ over
$[N]$ with densities at least $1/N$, construct spheres $A_1, \dotsc, A_k$ as in 
\cref{lemma:clustering}. Construct the map $\phi : [N] \to \bigcup_{i=1}^k A_i$ by sampling $s_1,
\dotsc, s_N \sim \cS$, where $\cS$ is the distribution from  \cref{lemma:clustering}, and
setting $\phi(i) = s_i$. Then simulate the tester for $k$-clusterability by giving the tester
samples $\phi(i)$ for $i \sim \cD$. We will write $\phi \cD$ for the distribution over
$\bigcup_{i=1}^k A_i$ obtained by sampling $i \sim \cD$ and returning $\phi(i)$.

First suppose that $|\supp(\cD)| \leq \alpha N$. Then $\supp(\phi \cD)$ is a set of at most $\alpha
N = nk/2$ points sampled from $\cS$, so by  \cref{lemma:clustering}, with probability at least
$9/10$ over the choice of $\phi$ the distribution $\phi \cD$ is $k$-clusterable, so the tester will
output 1 with probability at least $5/6$, so the total probability of success is at least $2/3$.

Next suppose that $|\supp(\cD)| \geq \beta N$ so $\supp(\phi \cD)$ is a set of between $\beta N =
4nk$ and $N=8nk$ points sampled from $\cS$. Then by  \cref{lemma:clustering}, for sufficiently
small constant $\epsilon > 0$, with probability at least $81/100$ over the choice of $\phi$, $X
\define \supp(\phi \cD)$ is at least $\epsilon/\beta$-far from $k$-clusterable according to the
uniform distribution over $X$. Since $\cD$ (and therefore $\phi \cD$) has densities at least $1/N$
on $X$, any $k$-clusterable distribution $\phi \cD$ must be at least $\frac{(\epsilon/\beta) |X|}{N}
\geq \epsilon$-far from $\phi \cD$. Therefore the $\epsilon$-tester will output 0 with probability
at least $5/6$, so the total probability to output 0 is at least $2/3$. So the algorithm solves
support-size distinction with parameters $N=8nk, \alpha=1/16, \beta = 1/2$.
Finally, by \cref{thm:wy easy application}, the number of samples required is at least
$\Omega\left(\frac{N}{\log N}\right) = \Omega\left(\frac{nk}{\log(nk)}\right)$.
\end{proof}

\subsection{Uniform Distributions and Testing LP-Type Problems}
\label{section:lp}

Epstein \& Silwal \cite{ES20} recently introduced property testing for LP-Type problems, which are
problems that generalize linear-programming. The algorithm has query access to a set $S$ of
constraints and must determine with high probability whether an objective function $\phi$ satisfies
$\phi(S) \leq k$ or if at least an $\epsilon$-fraction of constraints must be removed in order to
satisfy $\phi(S) \leq k$. We refer the reader to their paper for the definition of their model and
results in full generality, and describe only a special case here.

\begin{definition}[Testing Feasibility \cite{ES20}]
A tester for feasibility of a set of linear equations is an algorithm that performs as follows. On
an input set $S$ of linear equations over $\bR^n$, the algorithm samples equations $s \sim S$
uniformly at random, and must satisfy the following:
\begin{enumerate}
\item If $S$ is feasible, i.e.~there exists $x \in \bR^n$ that satisfies all equations $S$, then the
algorithm outputs 1 with probability at least $2/3$;
\item If at least $\epsilon |S|$ equations must be removed or flipped for the system to be feasible,
then the algorithm outputs 0 with probability at least $2/3$.
\end{enumerate}
\end{definition}
Epstein \& Silwal obtain a two-sided tester for this problem.
\begin{theorem}[\cite{ES20}]
There is a tester for feasibility in $\bR^n$ with two-sided error and sample complexity
$O(n/\epsilon)$.
\end{theorem}
Testing if a set $X \subseteq \bR^n$ with labels $\ell : X \to \pmset$ is realizable by
a halfspace can be solved by their algorithm, since for each $x \in X$ one can add the constraint
$\ell(x) \cdot (w_0 + \sum_{i=1}^n w_ix_i) \geq 1$ to $S$, with variables $w_0,w_1, \dotsc, w_n$.
On the other hand, they prove a lower bound for one-sided error:
\begin{theorem}[\cite{ES20}]
Testing with one-sided error whether a set $X \subseteq \bR^n$ with labels $\ell : X \to \pmset$ is
realizable by a halfspace or whether at least $\epsilon |X|$ labels must be changed to become
realizable by a halfspace requires at least $\Omega(d/\epsilon)$ samples.
\end{theorem}
\begin{remark}
\label{remark:ES}
\cite{ES20} does not specify that their lower bound is for one-sided error; however, their proof
relies on a claim that is true only for one-sided error \cite{Sil20}, namely that distinguishing
between uniform distributions with support size $d$ and uniform distributions with support size $3d$
requires at least $d+1$ samples -- with two-sided error, this can be done with only $O(\sqrt d)$
samples via a birthday paradox argument.
\end{remark}
We would like to prove lower bounds on two-sided error algorithms. However, our reduction from
support-size distinction will not work for this, because the model of LP testing uses the uniform
distribution as its distance measure, and the distributions that occur in the reduction are not
uniform. We can fix this by replacing the lower bound of Wu \& Yang \cite{WY19} with a weaker lower
bound of \cite{RRSS09} that uses distributions $\cD$ over $[n]$ with densities that are integer
multiples of $1/n$:

\begin{theorem} [\cite{RRSS09} Theorem 2.1]
\label{thm:rrss bound}
Let $\SSD^\bZ(n,\delta,1-\delta)$ be the support-size distinction problem under the promise that the
input distribution $\cD$ has densities that are integer multiples of $1/n$. Then for every $\delta
\geq 2\frac{\sqrt{\log n}}{n^{1/4}}$,
\[
  \SSD^{\bZ}(n,\delta,1-\delta) = \Omega(n^{1-\gamma}) \,,
\]
where $\gamma = 2\sqrt{\frac{\log(1/\delta) + \frac{1}{2}\log\log(n) + 1}{\log n}}$. In particular,
for constant $\delta$, the lower bound is $n^{1-o(1)}$.
\end{theorem}

We can now prove the following lower bound on testing linear separability:
\begin{theorem}
\label{thm:lp feasibility lower bound}
Testing with two-sided error whether a set $X \subseteq \bR^n$ with labels $\ell : X \to \pmset$ is
realizable by a halfspace or whether at least $\epsilon |X|$ labels must be changed to become
realizable by a halfspace requires at least $n^{1-o(1)}$ samples.
\end{theorem}
\begin{proof}
Repeat the proof of \cref{thm:main} and \cref{lemma:ssd reduction} with input
distributions $\cD$ over $[n]$ where for each $i \in \supp(\cD), \cD(i)$ is an integer multiple of
$1/n$. We obtain a set of points $X = \supp(\phi\cD) \subseteq \bR^n$ and labels $\ell : X \to
\pmset$ with integer probabilities, and we let $S$ be the set of linear constraints constructed from
$X,\ell$ as above, with each $x \in S$ occurring with multiplicity $t$ when $\cD(\phi^{-1}(x))=t/n$.
We may simulate samples from $S$ by samples from $\cD$. Therefore we obtain a lower bound of
$n^{1-o(1)}$ be the theorem of \cite{RRSS09}.
\end{proof}

\section{Upper Bounds}
\label{section:upper}

In this section we will prove upper bounds to complement the above lower bounds. First we study
symmetric classes of functions, which have been noted by Sudan \cite{Sud10} and Goldreich \& Ron
\cite{GR16} to be closely related to support size estimation. These classes establish the optimality
of \cref{thm:main}, and show that the lower bound for $(0,\epsilon)$-tolerant query testers cannot
be extended to intolerant testers, since symmetric classes exhibit a nearly-maximal separation
between tolerant and intolerant testing in the distribution-free setting.

Next, we show that there are natural classes of Boolean-valued functions, $k$-juntas and monotone
functions for which efficient distribution-free sample-based testing is possible.

\subsection{Symmetric Classes}
\label{section:optimality}

We first show that our lower bound in \cref{thm:main} is optimal, in the sense that there
exists a property $\cH$ where $d=\LVC_S(\cH)=\VC_S(\cH) \leq |S|/5$ and the sample complexity of
distribution-free testing is $\Theta(d / \log d)$. The upper bound will follow from a theorem of
Goldreich \& Ron \cite{GR16} for symmetric properties.

\begin{definition}
A set $\cH$ of functions $[n] \to \zo$ is \emph{symmetric} if for any permutation $\sigma : [n] \to
[n]$, for any function $f \in \cH$, it is also the case that $f \circ \sigma \in \cH$. Equivalently,
$\cH$ is symmetric iff there is a function $\phi : [n] \to \zo$ such that $f \in \cH$ iff
$\phi(k) = 1$ when $k = |\{ i \in [n] : f(i) = 1 \}|$.
\end{definition}

\begin{proposition}
Let $\cH$ be any symmetric class of functions $[n] \to \zo$. Then for any set $S \subseteq [n]$,
$\LVC_S(\cH) = \VC_S(\cH)$.
\end{proposition}
\begin{proof}
This follows from the fact that if $T \subseteq S$ is shattered by $\cH$, then every $T' \subseteq
S$ with $|T'|=|T|$ is also shattered.
\end{proof}

Symmetric properties are interesting because, as observed by Goldreich
\& Ron \cite{GR16}, there is a distribution-free testing upper bound for these sets that can be
obtained by the support-size estimation algorithm of Valiant \& Valiant \cite{VV11a,VV11b}. Together
with our lower bound, this shows that distribution-free testing symmetric sets $\cH$ is essentially
equivalent to deciding support size.

\begin{theorem}[Goldreich \& Ron \cite{GR16}, Claim 7.4.2]
\label{thm:gr upper bound}
For any symmetric class $\cH$ of functions $[n] \to \zo$,
$\mteste(\cH) = \poly(1/\epsilon) \cdot O\left(\frac{n}{\log n}\right)$.
\end{theorem}

On the other hand, consider the class $\cS_n$ of functions $[n] \to \zo$ such that $f : [n] \to \zo$
is in $\cS_n$ iff $|\{i \in [n] : f(i) = 1\}| \leq n/5$.
\begin{theorem}
\label{thm:tightness}
$\LVC_{[n]}(\cS_n) = \VC_{[n]}(\cS_n) = n/5$ and for small enough (constant) $\epsilon > 0$,
\[
  \mteste(\cS_n) = \Theta\left(\frac{n}{\log n}\right) \,.
\]
\end{theorem}
\begin{proof}
Any negative certificate for a function $f \notin \cS_n$ must have size at least $n/5+1$ so
$\LVC_{[n]}(\cS_n) \geq n/5$. On the other hand, any set $T$ of size $n/5$ is shattered since we may
assign 0 to all values $[n] \setminus T$. Therefore \cref{cor:easy bound} and
\cref{thm:gr upper bound} imply the conclusion.
\end{proof}

Next we show that the lower bound for $(0,\epsilon)$-tolerant adaptive testers cannot be extended to
intolerant testers.
Parnas, Ron, \& Rubinfeld \cite{PRR06} observed that, when testing over the uniform distribution,
any $\epsilon$-tester with uniformly (but not necessarily independently) distributed queries is in
fact $(\epsilon',\epsilon)$-tolerant for some $\epsilon' > 0$ (depending on the query cost). Since
our lower bound in \cref{thm:main} holds even for $(0,\epsilon)$-tolerant testers, one might
then wonder if it holds also for intolerant testers, in light of the observation of \cite{PRR06}.
However, this is not the case, and the counterexample is the same class of symmetric functions
discussed above. This class exhibits a nearly-maximal separation between tolerant and
intolerant testing in the distribution-free model, even when the intolerant tester has uniformly 
distributed and non-adaptive queries.
\begin{theorem}
There is a two-sided non-adaptive query tester for $\cS_n$ with query complexity
$O\left(\frac{1}{\epsilon^2}\right)$, and yet every adaptive $(0,\epsilon)$-tolerant tester for
$\cS_n$ has query complexity $\Omega\left(\frac{n}{\log n}\right)$ for small enough constant
$\epsilon > 0$.
\end{theorem}
\begin{proof}
The lower bound follows from \cref{cor:easy bound}. For the upper bound, consider the
algorithm that makes $m = \frac{50}{\epsilon^2}\ln(3)$ uniformly random samples, sets $X$ equal to
the number of sample points with value 1, and rejects iff $X > (1+\epsilon/2)\frac{m}{5}$.

Let $\cD,f$ be the input distribution and function, and suppose that $f \in \cS_n$. Then $\Ex{X}
\leq n/5$ so by Hoeffding's inequality,
\[
  \Pr{X > (1+\epsilon/2)\frac{m}{5}}
  \leq \Pr{X > \Ex{X}+\frac{\epsilon m}{10}} \leq \exp{-\frac{m\epsilon^2}{50}} \leq 1/3 \,.
\]
Now suppose that $\dist_\cD(f,\cH) > \epsilon$. Let $N = \{ x \in [n] : f(x)=1 \}$ and observe that
$|N| > n/5$. Write $\cD(x)$ for the probability density of $x$ according to $\cD$, let $A \subset
N$ be the $n/5$ points $x \in N$ with largest value $\cD(x)$, and let $B = N \setminus A$. Observe
that for all $x \in A, y \in B, \cD(x) > \cD(y)$, so the average $\cD(x)$ in $A$ is larger than the
average $\cD(x)$ in $B$. Write $\cD(A) \define \sum_{x \in A} \cD(x), \cD(B) \define \sum_{x \in B}
\cD(x)$. Since $\dist_\cD(f,\cH) > \epsilon$ it must be that $\cD(B) > \epsilon$ since otherwise the
function $f'$ obtained by flipping the values in $B$ is in $\cH$ and satisfies $\dist_\cD(f,\cH)
\leq \dist_\cD(f,f') \leq \epsilon$.
\[
  \frac{5}{n} \geq \frac{\cD(A)}{|A|} \geq \frac{\cD(B)}{|B|} \geq \frac{\epsilon}{|B|}
\]
so $|B| \geq \epsilon n / 5$. Therefore the number of 1-valued points according to $f$ is
$|N|=|A|+|B|\geq (1+\epsilon)\frac{n}{5}$, so $\Ex{X} \geq (1+\epsilon)\frac{m}{5}$. By
Hoeffding's inequality:
\[
  \Pr{X \leq (1+\epsilon/2)\frac{m}{5}}
  \leq \Pr{X \leq \Ex{X}-\frac{\epsilon m}{10}} \leq \exp{-\frac{m\epsilon^2}{50}} \leq 1/3 \,.
  \qedhere
\]
\end{proof}

\subsection{\texorpdfstring{$k$}{k}-Juntas}
\label{section:juntas}

A $k$-junta $\zo^n \to \zo$ on $n$ variables is a function that depends on only $k$ of the $n$
variables; these are of great interest in testing and learning because if a function depends on $k
\ll n$ variables then the complexity of learning may be significantly reduced. Blais \cite{Bla09}
gave a nearly optimal tester in the query model for product distributions, and Bshouty \cite{Bsh19}
recently presented a tester in the distribution-free query model, with query cost $\widetilde
O(k/\epsilon)$, but there are no known upper bounds in the sample-based distribution-free model. The
VC dimension of $k$-juntas is at least $2^k$ since any subcube of dimension $k$ can be shattered. We
prove a polynomial improvement over the VC dimension for distribution-free sample-based testers
when $k > \log\log n$ using the following version of the birthday problem.

\begin{proposition}
\label{prop:general birthday}
Let $p$ be any distribution over $[n]$. The probability that $m$ independent samples drawn from $p$
are all distinct is at most $e^{-\frac{(m-1)^2}{2n}}$.
\end{proposition}
\begin{proof}
It is known that the worst case probability distribution $p$ is uniform over $[n]$ \cite{Mun77}. For
the uniform distribution over $[n]$, the probability that all $m$ independent samples are distinct
is at most
\[
  \prod_{i=0}^{m-1} \left(1-\frac{i}{n}\right)
  \leq \prod_{i=0}^{m-1} e^{-\frac{i}{n}}
  = \exp{-\frac{1}{n}\sum_{i=0}^{m-1} i}
  = \exp{-\frac{m(m-1)}{2n}} \,. \qedhere
\]
\end{proof}

\begin{theorem}
\label{thm:juntas}
There is a distribution-free sample-based $\epsilon$-tester for $k$-juntas on domain
$\zo^n$ with one-sided error and sample complexity $O\left(\frac{k2^{k/2}\log(n/k)}{\epsilon}\right)$.
\end{theorem}
\begin{proof}
For a set $S \subseteq [n]$ of size $n-k$ we will arrange the points $x \in \zo^n$ into ``rows'' and
``columns''; for every partial assignment $\rho : \overline S \to \zo$ let row $R_\rho$ be the
set of points $x \in \zo^n$ such that $\forall i \notin S, x_i = \rho(i)$, and for every partial
assignment $\gamma : S \to \zo$ let column $C_\gamma$ be the set of points $x \in \zo^n$ such that
$\forall i \in S, x_i = \gamma(i)$.

The tester is as follows: On input $f : \zo^n \to \zo$ and distribution $p$, sample a set $Q$ of $s
\cdot m$ points, where $s = O\left(\log{n \choose k}\right)$ and $m =
O\left(\frac{2^{k/2}}{\epsilon}\right)$; since ${n \choose k} \leq \left(\frac{e n}{k}\right)^k$,
the sample complexity is $sm=O\left(\frac{k2^{k/2}\log(n/k)}{\epsilon}\right)$. Reject if for every
set $S \subset [n]$ of $n-k$ variables, there exists a row $\rho : \overline S \to \zo$ that
contains $x,y \in Q \cap R_\rho$ such that $f(x) \neq f(y)$; we will call such a pair $x,y$ a
\emph{witness} for $S$. This has one-sided error because a $k$-junta has a set $S$ of variables such
that $f$ is constant on every row.

Let $p : \pmset^n \to \bR$ be a probability distribution over $\pmset^n$.  Let $f : \zo^n \to \zo$
and let $S \subset [n]$ be a set of $n-k$ variables. For each $\rho : \overline S \to \zo$ write
\begin{align*}
r_\rho^0 \define \sum_{x \in R_\rho : f(x)=0} p(x) &&
r_\rho^1 \define \sum_{x \in R_\rho : f(x)=1} p(x)
\end{align*}
Suppose that
\[
  \sum_{\rho : \overline S \to \zo} \min(r_\rho^0,r_\rho^1) < \epsilon \,.
\]
Then $f$ is $\epsilon$-close to a $k$-junta, because we can define the $k$-junta $h$ as follows. For
each $x$ we can set $h(x)=0$ if $r_\rho^0 \geq r_\rho^1$ and $h(x)=1$ if $r_\rho^0 < r_\rho^1$,
where $\rho$ is the partial assignment defining the row $R_\rho$ containing $x$. Since $h$ is
constant on each row, it does not depend on any of the variables that are not assigned by $\rho :
\overline S \to \zo$, i.e.~it does not depend on any of the $n-k$ variables in $S$. And by
definition,
\[
\dist_p(f,h) = \sum_{\rho : \overline S \to \zo} \min(r_\rho^0,r_\rho^1) < \epsilon \,.
\]
If $f$ is $\epsilon$-far, then every set $S$ of $n-k$ variables satisfies
\[
  \sum_{\rho : \overline S \to \zo} \min(r_\rho^0,r_\rho^1) \geq \epsilon \,.
\]
For any fixed $S$, we can bound the probability that the set $Q$ does not contain any witness as
follows. Without loss of generality assume that $r_\rho^0 \leq r_\rho^1$ for every $\rho$, and
choose a set $T_\rho \subseteq R_\rho$ such that
\[
  r_\rho^0 = \sum_{x \in T_\rho : f(x)=0} p(x) = \sum_{x \in T_\rho : f(x)=1} p(x) \,,
\]
which we may do since, without loss of generality, we may adjust the probabilities $p(x)$ in each
row without changing the probability of finding a witness, as long as the totals $r_\rho^0,r_\rho^1$
are invariant. Note that if two random points $x,y \sim p$ fall in $T_\rho$, then with probability
$1/2$ we will have $f(x) \neq f(y)$. Therefore
\[
  \Pru{Q}{ \exists x,y,\rho : f(x)\neq f(y), x,y \in R_\rho }
  \geq \frac{1}{2} \cdot \Pru{Q}{\exists \rho : T_\rho \text{ contains $\geq 2$ points}} \,.
\]
Let $T = \cup_{\rho : \overline S \to \zo} T_\rho$ and observe that $\sum_{x \in T} p(x) = \sum_\rho
r_\rho^0 \geq \epsilon$, so in expectation there are $\epsilon m$ points in $Q \cap T$. By the
Chernoff bound,
\[
  \Pr{ |Q \cap T| < \frac{\epsilon m}{2}} \leq \exp{-\frac{\epsilon m}{8}} = o(1) \,.
\]
Assume there are at least $\epsilon m / 2$ points in $T$. By \cref{prop:general
birthday}, the probability that no $T_\rho$ contains at least 2 points is, for $N=2^k$ being the
number of rows, at most
\[
  \exp{-\frac{\left(\frac{\epsilon m}{2}-1\right)^2}{2N}} < \frac 1 2 \,,
\]
since $m = \Omega\left(\frac{\sqrt N}{\epsilon}\right)$. Therefore the probability of finding a
witness for $S$ is at least
\[
  \frac{1}{2} \cdot \Pru{Q}{\exists \rho : T_\rho \text{ contains $\geq 2$ points}}
  \geq \frac{1}{2} \cdot (1-o(1)) \cdot \frac{1}{2} = (1-o(1)) \frac 1 4 > \frac 1 5 \,.
\]
By repeating the sampling procedure $s = O\left(\log{n \choose k}\right)$ times, the probability of
failing to find a witness is at most $(4/5)^s < \frac{1}{3} {n \choose k}^{-1}$. Then by the union
bound, the probability that there exists $S$ on which the tester fails to find a witness is at most
$1/3$, since there are at most ${n \choose k}$ such sets.
\end{proof}

\subsection{Monotonicity in General Posets}
\label{section:monotonicity}

A basic result in testing monotonicity of Boolean functions over the uniform distribution is that
at most $O(\sqrt{n / \epsilon})$ uniform samples are necessary for any partial order of size $n$
\cite{FLN+02}. We extend this result to the distribution-free setting. The VC
dimension of the class of monotone functions over any poset $P$ is the width, i.e.~the size of the largest
antichain in $P$. For example, the standard partial ordering of the hypercube $\zo^n$ has width
$\Theta(2^n/\sqrt n)$ since the set of points $x \in \zo^n$ with Hamming weight $n/2$ is an
antichain of size ${n \choose n/2}$. Therefore, for the hypercube, distribution-free sample-based
testing can be done with sample complexity $O(2^{n/2}) = \widetilde O(\sqrt \VC)$.

For sets $X,Y$ and a set of order pairs $E \subseteq X \times Y$, we call the triple $(X,Y,E)$ a
\emph{bipartite} partial order, where the edges $E$ define the following partial order on $X \cup Y:
x < y$ iff $(x,y) \in E$. Fischer \emph{et al.}~\cite{FLN+02} observed that for the uniform
distribution, monotonicity on general finite posets reduces to testing on bipartite posets; we
generalize their reduction to the distribution-free setting:
\begin{lemma}
\label{lemma:bipartite reduction}
If for every bipartite partial order $(X,Y,E)$ of size $|X|=|Y|=n$ there is a distribution-free
sample-based $\epsilon$-tester for monotonicity with sample complexity $m(n,\epsilon)$ then for
every partial order $P$ of size $|P|=n$ there is a distribution-free sample-based $\epsilon$-tester
for monotonicity with sample complexity $m(n,\epsilon/2)$.
\end{lemma}
\begin{proof}
On any partial order $P$ with distribution $p$ and input function $f : P \to \zo$, consider the
following reduction: let $X,Y$ be separate copies of $P$ and for each $x \in P$ write $x_1,x_2$ for
the copies of $x$ in $X,Y$ respectively. Define a set of edges $E \subset X
\times Y$ where $(x_1,y_2) \in E$ iff $x < y$ in $P$. Define the distribution $q$ over $X \cup Y$ as
$q(x_1)=\frac{1}{2}p(x)$ for each $x \in X$ and $q(y_2)= \frac{1}{2}p(y_2)$ for each $y \in Y$.
Define the function $g : X \cup Y \to \zo$ as $g(x_1)=f(x),g(y_2)=f(y)$ for each $x_1 \in X, y_2 \in
Y$.  Observe that we can simulate a random sample from $q$ labelled by $g$ by sampling $x \sim p$
and taking $x_1,x_2$ with equal probability, labelling it with $f(x)$.

It is clear that if $f$ is monotone on $P$ then $g$ is monotone on $(X,Y,E)$, since $(x_1,y_2) \in
E$ implies $x < y$ in $P$. Suppose now that $g$ is $\epsilon$-close to monotone in $(X,Y,E)$
according to $q$, and let $h$ be monotone on $(X,Y,E)$ minimizing distance to $g$. Define $f' : P
\to \{0,1,*\}$ as follows: For $x \in P$, if $h(x_1)=h(x_2)=f(x)$ set $f'(x)=f(x)$, and otherwise
set $f'(x)=*$. Then
\begin{align*}
\sum_{x \in P : f'(x)=*} p(x)
  &= \sum_{x \in P} \ind{h(x_1) \neq f(x) \vee h(x_2) \neq f(x)} p(x) \\
  &\leq \sum_{x \in P} \left(\ind{h(x_1) \neq f(x)} + \ind{h(x_2) \neq f(x)}\right) p(x) \\
  &= \sum_{x \in P} \ind{h(x_1) \neq g(x_1)} p(x) + \sum_{x \in P} \ind{h(x_2) \neq g(x_2)} p(x) \\
  &= 2\sum_{x \in P} \ind{h(x_1) \neq g(x_1)} q(x_1) + 2\sum_{x \in P} \ind{h(x_2) \neq g(x_2)} q(x_2) \\
  &= 2\dist_q(g,h) < 2\epsilon \,.
\end{align*}
Now construct a monotone function $f'' : P \to \zo$ as follows. Take any total order $\prec$
consistent with the partial order on $P$. For each $x \in P$ in order of $\prec$, if $f'(x)=*$ set
$f''(x) = \max_{y < x} f''(y)$, otherwise set $f''(x)=f'(x)=f(x)$. Then $\dist_p(f,f'') \leq \sum_{x
\in P : f'(x)=*} p(x) < 2\epsilon$. Suppose that $f''$ is not monotone, so there are $x < y$ such
that $f''(x)=1,f''(y)=0$; assume $x$ is a minimal point where this occurs. Since $x \prec y$ it must
be the case that $f'(y) \neq *$, so $f'(y)=f(y)=0$. Since $f'$ is monotone except on $*$-valued
points, it must be that $f'(x)=*$, and $1=f''(x)=\max_{z < x} f''(z)$. But then $z < x$ and
$f''(z)=1,f''(y)=0$, so $x$ was not minimal, a contradiction. Thus $f''$ is monotone and
$\dist(f,f'') < 2\epsilon$.

We therefore conclude that if $f$ is $\epsilon$-far from monotone on $P$ according to $p$ then $g$
is at least $(\epsilon/2)$-far from monotone on $(X,Y,E)$ according to $q$. Therefore, by simulating
the distribution-free one-sided sample-based tester on $(X,Y,E)$ with parameter $\epsilon/2$ we
obtain a distribution-free one-sided tester for $P$.
\end{proof}

\begin{theorem}
\label{thm:monotonicity upper bound}
For any finite partial order $P$ of size $|P|=n$, there is a distribution-free, one-sided,
sample-based $\epsilon$-tester for monotonicity with sample complexity $O\left(\frac{\sqrt
n}{\epsilon}\right)$.
\end{theorem}
\begin{proof}
By \cref{lemma:bipartite reduction}, it suffices to consider bipartite partial orders.
Let $(X,Y,E)$ be a bipartite partial order.  On input $f : X \cup Y \to \zo$ and distribution $p : X
\cup Y \to \bR$, the tester will sample a set $Q$ of $m=O\left(\frac{\sqrt n}{\epsilon}\right)$
points from $p$ and reject if there exist $x \in X \cap Q, y \in Y \cap Q$ such that $x < y$ and
$f(x)=1,f(y)=0$; we call such a pair a \emph{violating pair}. 

Suppose $f$ is $\epsilon$-far from monotone. Let $X_1 \define \{ x \in X : f(x)=1\}$ and $Y_0
\define \{ y \in Y : f(y) = 0\}$.
For each $x \in X_1$ let $V_x \define \{ y \in Y_0 : x < y \}$ be the set of points $y \in Y_0$ such
that $(x,y)$ is a violating pair, and define $q(x) \define \sum_{y \in V_x} p(y)$ be the total
probability mass of all the points $y$ such that $(x,y)$ is a violating pair. Suppose for
contradiction that
\[
  \sum_{x \in X_1} \min(p(x),q(x)) < \epsilon \,.
\]
Construct a monotone function $h$ as follows. For each $x \in X_1$ (in arbitrary order), if $p(x) <
q(x)$ set $h(x)=0$, otherwise set $h(y)=1$ for all $y \in V_x$. The resulting function is now
monotone since each violating pair $(x,y)$ now has either $h(x)=0$ or $h(y)=1$. The distance between
$f$ and $h$ increases by at most $\min(p(x),q(x))$ for each $x \in X_1$, so $\dist_p(f,h) <
\epsilon$, a contradiction. Therefore we must have $\sum_{x \in X_1} \min(p(x),q(x)) \geq \epsilon$.

Define a new distribution $r$ on $X \cup Y$ that is initialized to $r=p$ but then is updated to set $r(x) = \min(p(x),q(x))$ for each $x \in
X_1$, reassigning the remaining probability mass $p(x) - r(x)$ to an arbitrary point not in $X_1 \cup Y_0$ (we may
assume such a point exists since otherwise $f$ is the trivial function where every pair $x < y$ is a
violating pair). This reassignment can only decrease the probability of finding a violating pair in
the sample $Q$. Under the new distribution $r(X_1) = \sum_{x \in X_1} r(x) = \sum_{x \in X_1}
\min(p(x),q(x)) \geq \epsilon$, and for each $x \in X_1, r(x) \leq r(V_x)$. Let $R$ be a set of $m$
independent points drawn from $r$, so $\Pr{Q \text{ contains a violating pair}} \geq \Pr{R \text{
contains a violating pair}}$.

Now observe that for each $x \in X_1$, the probability that $R$ contains some violating pair $(x,y)$
is at least the probability that $x$ occurs twice in $R$; this is because $r(V_x) = q(x) \geq r(x)$.
We will now bound the probability that there exists $x \in X_1$ that occurs twice in $R$. The
expected number of points in $R \cap X_1$ is $\Ex{|R \cap X_1|} = m \sum_{x \in X_1} r(x) \geq
\epsilon m$.  By the Chernoff bound,
\[
  \Pr{|R \cap X_1| < \frac{\epsilon m}{2}} \leq \exp{-\frac{\epsilon m}{8}} = o(1) \,.
\]
Assuming $|R \cap X_1| \geq \frac{\epsilon m}{2}$, by  \cref{prop:general birthday}, the
probability that each $x \in R \cap X_1$ occurs at most once in $R$ is at most
\[
  \exp{-\frac{(|R \cap X_1|-1)^2}{2|X_1|}}
  \leq \exp{-\frac{((\epsilon m/2)-1)^2}{2n}}
  < 1/4 \,,
\]
for sufficiently large $m = O\left(\frac{\sqrt n}{\epsilon}\right)$. Therefore
\[
  \Pr{\exists x \in X_1, x \text{ occurs at least twice in } R}
  \geq 1 - \frac 1 4 - o(1) \geq \frac 2 3 \,.
\]
Finally,
\begin{align*}
  \Pr{Q \text{ contains a violating pair}}
  &\geq \Pr{ R \text{ contains a violating pair}} \\
  &\geq \Pr{ \exists x \in X_1, x \text{ occurs at least twice in } R}
  \geq \frac 2 3 \,. \qedhere
\end{align*}
\end{proof}

\section{Remarks on Testing with One-Sided Error}
\label{section:lvc}

In this section we briefly discuss the relationship of the VC and LVC dimensions to one-sided
testing, and using a theorem of Goldreich \& Ron \cite{GR16} that relats one- to two-sided error we
establish a more general but weaker lower bound on two-sided testing via LVC dimension.

In \cref{section:lower bound}, we motivated the definition of LVC dimension by observing that a
property tester can reject as soon as it finds a certificate of non-membership. Unlike two-sided
error testers, one-sided testers \emph{must} find a certificate of non-membership, so LVC is also an
essential quantity for one-sided testers; this is made formal in the next proposition:
\begin{proposition}
\label{prop:one-sided lower bound}
Let $\cH$ be a set of functions $\cX \to \zo$, let $\epsilon > 0$, and let $\cD$ be any distribution
over $\cX$ such that there exists $f : \cX \to \zo$ with $\dist_\cD(f,\cH) > \epsilon$. Then any
one-sided $\epsilon$-tester for $\cH$ over $\cD$ (even adaptive testers using queries)
requires at least $\LVC_S(\cH)$ queries, where $S$ is the support of $\cD$, under the assumption
that all queries fall within $S$ (which holds in particular for sample-based testers).
\end{proposition}
\begin{proof}
Suppose $A$ is any algorithm that makes at most $q$ queries, where $q \leq \LVC_S(\cH)$, and for any
function $f : \cX \to \zo$ let $\cQ_f$ be the distribution of query sequences $((x_1,f(x_1)),
\dotsc, (x_q,f(x_q)))$ made by the algorithm on input $f$. Since the algorithm has one-sided error,
it must accept every sequence $Q_h \sim \cQ_h$ with probability 1 when $h \in \cH$. Consider any $f
: \cX \to \zo$ and any sequence $Q_f \in \supp (\cQ_f)$. Since $q \leq \LVC_S(\cH)$ and
each $x_i \in S$, the set $\{x_1, \dotsc, x_q\}$ is shattered by $\cH$, so there exists $h \in \cH$
such that $h(x_i)=f(x_i)$ for each $i$; therefore there is $Q_h \in \supp(\cQ_h)$ such that
$Q_h=Q_f$. Then for every $f, Q_f \sim \cQ_f$ is accepted with probability 1, a contradiction.
\end{proof}

Goldreich \& Ron \cite{GR16} prove the following relationship between one- and two-sided error
testers.

\begin{theorem}[\cite{GR16}, Theorem 1.3 part 1]
\label{thm:gr16 general upper bound}
For every class $\cH$ of functions $\cX \to \zo$ with $\cX$ finite, if there is a distribution-free
sampling $\epsilon$-tester for $\cH$ using $q(\epsilon)$ samples, then there is a one-sided error
sampling $\epsilon$-tester for $\cH$ over the uniform distribution on $\cX$ using at most
$\widetilde O(q(\epsilon)^2)$ samples.
\end{theorem}

Using this relationship between one- and two-sided error testers, and the lower bound on one-sided
error in terms of LVC dimension, we get a general relationship between two-sided testers and LVC
dimension.
\begin{corollary}
\label{cor:gr lower bound}
Let $\cH$ be a class of functions $\cX \to \zo$ where $\cX$ is finite, let $\epsilon > 0$, and let
$S \subseteq \cX$ be such that there exists a distribution $\cD$ supported on $S$ and a function $f
: \cX \to \zo$ satisfying $\dist_\cD(f,\cH) > \epsilon$.  Then
\[
  \mteste(\cH) = \widetilde \Omega(\sqrt{\LVC_S(\cH)}) \,.
\]
\end{corollary}
\begin{proof}
This follows from \cref{thm:gr16 general upper bound} and  \cref{prop:one-sided
lower bound}.
\end{proof}
Our main result shows that for classes with large LVC dimension, two-sided error testing is not much
more efficient than learning. We would like to say that two-sided error also does not give a
significant advantage over one-sided error, so we want an upper bound on one-sided testing in terms
of the VC dimension.  A well-known relationship between testing and (proper) PAC learning
\cite{GGR98} says that testing is easier than learning. After running the learning algorithm as a
black box and using $O(1/\epsilon)$ additional samples, the algorithm can then reject $f$ if it is
not $\epsilon$-close to the learned function $h$. Since proper PAC learning requires
$O(\VC(\cH)\tfrac{1}{\epsilon}\log\tfrac{1}{\epsilon})$ samples (see \cite{SSBD14}), this gives the
same upper bound for testing, but this black-box algorithm has two-sided error since the learning
algorithm may fail. We can modify the PAC learning upper bound to give a one-sided error testing
upper bound; we remark that this result is likely not new, though we have not found a reference for
it (a similar proof with a weaker bound was presented in \cite{ADPR03}).
\begin{lemma}
\label{lemma:one-sided upper bound}
Let $\cH$ be a set of functions $\cX \to \zo$ with finite $d = \VC(\cH) > 0$. Then for any
$\epsilon > 0$, $\mone_\epsilon(\cH) = O\left(\frac{d}{\epsilon}\log\frac{1}{\epsilon}\right)$.
\end{lemma}
\begin{proof}
Let $f : \cX \to \zo$ be the input function, and let $\cD$ be the input distribution over $\cX$.
The algorithm is as follows. Draw a set $S$ of $m =
O\left(\frac{d}{\epsilon}\log\frac{1}{\epsilon}\right)$ labelled examples from $\cD$ and
accept if there exists $h \in \cH$ such that $f(x)=h(x)$ for all $x \in S$; otherwise, reject.

If $f \in \cH$ then this algorithm accepts with probability 1, so assume that $f$ is $\epsilon$-far
from $\cH$. Define the class $f \oplus \cH \define \{ f \oplus h : h \in \cH\}$ and observe that
a set is shattered by $f \oplus \cH$ iff it is shattered by $\cH$. By standard VC dimension
arguments (e.g.~\cite{SSBD14} Theorem 28.3), with probability at least $2/3$ a sample $S$ of size
$m$ is an $\epsilon$-net for $f \oplus \cH$, meaning that for every $h \in \cH$, if $\Pru{x \sim
\cD}{f(x) \neq h(x)} = \Pru{x \sim \cD}{(f \oplus h)(x)=1} \geq \epsilon$ then there exists $x \in
S$ such that $(f \oplus h)(x)=1$, i.e.~$f(x) \neq h(x)$. Since $\Pru{x \sim \cD}{f(x) \neq h(x)}
\geq \epsilon$ for every $h \in \cH$, this implies that the algorithm rejects.
\end{proof}

\newpage
\appendix

\section{Glossary}
\label{glossary}

We include a summary of the (standard) property testing definitions used in the paper in this section for the convenience of the reader.

\begin{definition}
Let $\cH$ be a set of functions $\cX \to \zo$ and $\epsilon > 0$. A \emph{distribution-free sample
$\epsilon$-tester} for $\cH$ with \emph{sample complexity} $m$ is a (randomized) algorithm $A$ that,
for any input distribution $\cD$ over $\cX$, input function $f : \cX \to \zo$, receives a sequence
$S$ of $m$ independent labelled examples $(x,f(x))$ where $x \sim \cD$, and outputs $A(S) = 0$ or
$A(S)=1$, such that, for any $f,\cD$:
\begin{itemize}
\item If $f \in \cH$ then $\Pru{A,S}{A(S)=1} \geq 2/3$; and,
\item If $\dist_\cD(f,\cH) > \epsilon$ then $\Pru{A,S}{A(S)=0} \geq 2/3$.
\end{itemize}
We write $\mteste(\cH)$ for the minimum number $m$ such that there exists a distribution-free
sample $\epsilon$-tester for $\cH$ with sample complexity $m$. A distribution-free sample tester for
$\cH$ is \emph{one-sided} if for all $f \in \cH, \Pru{S}{A(S)=1}=1$. We will write
$\mone_\epsilon(\cH)$ for the minimum $m$ such that there exists a one-sided distribution-free
sample $\epsilon$-tester with sample complexity $m$.
\end{definition}

\begin{definition}[Tolerant Testing]
Let $\cH$ be a set of functions $\cX \to \zo$ and $\epsilon_2 > \epsilon_1 \geq 0$. A
\emph{$(\epsilon_1,\epsilon_2)$-tolerant distribution-free tester} for $\cH$ with \emph{query
complexity} $q$ is a randomized algorithm $A$ that, for any input distribution $\cD$ and function $f
: \cX \to \zo$, may query $f(x)$ at arbitrary points $x \in \cX$ or sample $(x,f(x))$ for $x \sim
\cD$, where the total number of values $f(x)$ queried or sampled is at most $q$. Write $A(f,\cD)$
for the (random) output of $A$ on inputs $f,\cD$, which must satisfy the
following:
\begin{itemize}
\item If $\dist_\cD(f,\cH) \leq \epsilon_1$ then $\Pru{A}{A(f,\cD) = 1} \geq 2/3$; and,
\item If $\dist_\cD(f,cH) \geq \epsilon_2$ then $\Pru{A}{A(f,\cD) = 0} \geq 2/3$.
\end{itemize}
Note that we allow the algorithm $A$ to be \emph{adaptive}, i.e.~its choice of query $x$ may depend
on the answers to previous queries (or samples). We will write $\qtole(\cH)$
for the minimum $q$ such that there exists an $(\epsilon_0,\epsilon_1)$-tolerant, distribution-free,
adaptive query tester for $\cH$. It is clear that $\qtol_{\epsilon_0,\epsilon_1}(\cH) \geq
\qtol_{0,\epsilon_1}(\cH)$ for any $\epsilon_0 < \epsilon_1$.
\end{definition}

\begin{definition}[Active Testing \cite{BBBY12}]
Let $\cH$ be a set of functions $\cX \to \zo$ and $\epsilon > 0$. A \emph{distribution-free active
$\epsilon$-tester} with sample complexity $m$ and query complexity $q$ is a randomized algorithm $A$
that, for any input distribution $\cD$ and (measurable) function $f : \cX \to \zo$, receives a set
$S$ of $m$ independent (unlabelled) samples from $\cD$, and then (adaptively) queries $f(x)$ on $q$
points $x \in S$. Write $A(f,S)$ for the output of $A$ on input $f$ and sample set $S$. $A$ must
satisfy:
\begin{itemize}
\item If $f \in \cH$ then $\Pru{S,A}{A(f,S)=1} \geq 2/3$; and,
\item If $\dist_\cD(f,\cH) \geq \epsilon$ then $\Pru{S,A}{A(f,S)=0} \geq 2/3$.
\end{itemize}
\end{definition}

\section{Lower Bound on Support Size Distinction}

We provide an exposition of the proof of \cref{thm:wy easy application} in this section. The proof that follows is a very slight adaptation of the proof of Wu \& Yang~\cite{WY19}, with the only 
changes to their proof being the ones necessary to adapt the 
lower bound to be on the decision problem of support size distinction (SSD) instead of the support size estimation (SSE) problem.

We begin by 
defining a decision problem for distributions-of-distributions over $\bN$.

\begin{definition}[Meta-Distribution Decision Problem]
Let $\cP,\cQ$ be two distributions over probability distributions on $\bN$.
$\DEC(\cP,\cQ)$ is the minimum number $m$ such that there exists an algorithm $A$ that draws a set
$S$ of $m$ independent samples from its input distribution, and its output $A(S)$ satisfies the
following:
\begin{itemize}
\item $\Pru{p \sim \cP, S \sim p^m}{A(S) = 1} \geq 2/3$; and,
\item $\Pru{q \sim \cQ, S \sim q^m}{A(S) = 0} \geq 2/3$.
\end{itemize}
\end{definition}

It is clear that \cref{thm:wy easy application} follows from the fact that, for any $0 <
\alpha < \beta \leq 1$ such that $\alpha \geq \delta$ and $\beta \leq 1-\delta$,
\[
  \SSD(n,\alpha,\beta) \geq \sup_{\cP,\cQ} \DEC(\cP,\cQ) \,,
\]
where the supremum is taken over all distributions $\cP,\cQ$ over distributions on $[n]$ such that
any $p \in \supp(\cP)$ has $|\supp(p)| \leq \delta n \leq \alpha n$, any $q \in \supp(\cQ)$ has
$|\supp(q)| \geq (1-\delta) n \geq \beta n$, and any $p \in \supp(\cP) \cup \supp(\cQ)$ has $p_i
\geq 1/n$ for each $i \in \supp(p)$. Therefore, to establish  \cref{thm:wy easy application},
it suffices to prove the following theorem.
\label{section:wy}
\begin{theorem}[\cite{WY19}]
\label{thm:wy adaptation}
There is a constant $C > 0$ such that,
for every $n \in \bN$ and every $C\frac{\sqrt{\log n}}{n^{1/4}} < \delta < \tfrac12$,
there exist distributions $\cP,\cQ$ over the space of probability
distributions on $[n]$, such that $\DEC(\cP,\cQ) = \Omega\left(\frac{n}{\log
n}\log^2\frac{1}{1-\delta}\right)$, and where:
\begin{itemize}
\item Every $p \in \cP$ has support size at most $\delta n$;
\item Every $p \in \cQ$ has support size at least $(1-\delta) n$;
\item Every $p \in \cP \cup \cQ$ has $p(x) \geq 1/n$ for all $x \in \supp(p)$.
\end{itemize}
\end{theorem}
What follows is adapted from the proofs of Wu \& Yang \cite{WY19}.

\begin{definition}
For any $\nu \geq 0$, let $\cD_n(\nu)$ be the set of vectors $p \in \bR^n$ such that each
$p_i$ satisfies $p_i \in \{0\} \cup \left[ \frac{1+\nu}{n}, 1\right]$, and $\left| 1 - \sum_i
p_i \right| \leq \nu$. For $p \in \cD_n(\nu)$, we will write $\supp(p) = \{i \in [n] : p(i) > 0\}$.
Note that $\cD_n(0)$ is the set of probability distributions over $[n]$ with densities at least
$1/n$ on the support.
\end{definition}

\begin{definition}[Poisson Sampling Model]
Let $\cP,\cQ$ be distributions over $\cD_n(\nu)$.
Define $\widetilde{\DEC}(\cP,\cQ)$ as the smallest number $m$ such that there is an algorithm $A$
that does the following on input $p \in \cD_n(\nu)$: For each $i \in \bN$, $A$ receives a vector $s
: \bN \to \bN$ such that $s(i) \sim \Poi(mp(i))$. $A(s)$ outputs 0 or 1, and
satisfies:
\begin{itemize}
\item $\Pru{p \sim \cP,s}{A(s) = 1} \geq 2/3$;
\item $\Pru{q \sim \cQ,s}{A(s) = 0} \geq 2/3$.
\end{itemize}
\end{definition}

\begin{lemma}
\label{lemma:wy poissonization}
For any $n,\nu$, suppose that $\cP,\cQ$ are distributions over $\cD_n(\nu)$. Then for distributions
$\cP', \cQ'$ over $\cD_n(0)$ defined by choosing $p \sim \cP$ and taking $p / \sum_i p(i)$, or by
choosing $q \sim \cQ$ and taking $q / \sum_i q(i)$, respectively,
\[
  \DEC(\cP',\cQ') \geq \Omega((1-\nu) \cdot \widetilde{\DEC}(\cP,\cQ)) \,.
\]
\end{lemma}
\begin{proof}
First we observe that for any $t$, conditioned on the event $\sum_i s_i = t$, the vector $s$ has the
same distribution as the vector $s'$ obtained by drawing $t$ points $i \in \bN$ independently from
$p / \sum_i p(i)$ and letting $s'_i$ be the number of times item $i$ is observed.

For any $k$, let $A_k$ be the algorithm that, receiving $k$ independent samples from the input
distribution, has the highest probability of correctly distinguishing $\cP'$ from $\cQ'$. Let $m$ be
the minimum number such that $A_m$ has success probability at least $9/10$, so that
\[
  \Pru{p \sim \cP',s'}{A_m(s') = 1} \geq 9/10 \qquad
  \Pru{q \sim \cQ',s'}{A_m(s') = 0} \geq 9/10 \,.
\]
Observe that (by standard boosting techniques), $m = \Theta(\DEC(\cP', \cQ'))$.  For some $m' = \rho
m$ (with $\rho > 1$ to be chosen later), we construct an algorithm in the Poisson testing model
where $s_i \sim \Poi(m'p(i))$ and upon receiving a vector $s$ with $\sum_i s_i = t$, runs $A_t(s)$.
\begin{align*}
\Pru{p \sim \cP, s}{A(s)=0}
&= \sum_{k=0}^\infty \Pr{t=k} \Pruc{p,s}{A_k(s)=0}{\sum_i s_i = k} \\
&= \sum_{k=0}^\infty \Pr{t=k} \Pruc{p',s'}{A_k(s')=0}{\sum_i s_i = k} \\
&\leq \frac{1}{10} \Pr{t \geq m} + \Pr{t < m} = \frac{1}{10} + \frac{9}{10}\Pr{t < m} \,.
\end{align*}
The same argument shows that for $q \sim \cQ$,
\[
  \Pru{q \sim \cQ, s}{A(S)=1} \leq \frac{1}{10} + \frac{9}{10}\Pr{t < m} \,,
\]
so what remains is to bound $m'$.  $t$ is a sum of independent Poisson random variables
$\Poi(m'p_i)$, so $t \sim \Poi(m'\sum_i p(i))$, which has mean $m'\sum_i p(i) \geq m'(1-\nu) =
(1-\nu)\rho m$. For $X \sim \Poi(\lambda)$ and $z < \lambda$ we use the inequality:
\[
  \Pr{X < z} \leq \frac{(e\lambda)^z e^{-\lambda}}{z^z} \,,
\]
which implies
\[
  \Pr{t < m} \leq \frac{(e(1-\nu)\rho m)^m e^{-m\rho(1-\nu)}}{m^m}
    = (e(1-\nu)\rho)^m e^{-m\rho(1-\nu)}
    = \exp{m(\ln(e(1-\nu)\rho)-\rho(1-\nu))} \,.
\]
For any constant $C > 0$ there is $C'$ such that for $\rho > C'/(1-\nu)$, this probability is at
most $\exp{m(1+\ln(C')-C')} = \exp{-Cm}$, so we can choose $\exp{-C} < 1/100$ to obtain a total
failure probability of at most $1/10 + 9/100 < 1/3$, with $m' = \rho m = O(m/(1-\nu))$. Thus
\[
  \widetilde{\DEC}(\cP,\cQ) \leq m' = \frac{m}{1-\nu}
    = O\left(\frac{1}{1-\nu} \cdot \DEC(\cP',\cQ')\right) \,. \qedhere
\]
\end{proof}

\begin{lemma}
\label{lemma:wy poisson lower bound}
Let $\nu, \lambda > 0$, and
Suppose $P,Q$ are random variables taking values in $\{0\} \cup [1+\nu,\lambda]$, such that
$\Ex{P}=\Ex{Q}=1, \Ex{P^j}=\Ex{Q^j}$ for all $j \in [L]$, and $|\Pr{P > 0}-\Pr{Q > 0}| = \delta$.
Then for any $\alpha < 1/2$, if
\[
  \frac{2\lambda}{n\nu^2} + \frac{2}{n\alpha^2\delta^2} + n\left(\frac{em\lambda}{2nL}\right)^L
  < 1/3 \,,
\]
then there exist distributions $\cP,\cQ$ over $\cD_n(\nu)$ such that
$\widetilde{\mathsf{DEC}}(\cP,\cQ) \geq m$, and for each $p \in \supp(\cP), q
\in \supp(\cQ)$, $\#\supp(q)-\#\supp(p) \geq (1-2\alpha)\delta n$ and $p(x),q(x) > (1+\nu)/n$ for
each $x \in \supp(p),\supp(q)$ respectively.
\end{lemma}
\begin{proof}
Let $\cP'$ be the distribution over vectors $\bR^n$ obtained by drawing $p \sim
\frac{1}{n}(P_1, \dotsc, P_n)$ where each $P_i$ is an independent copy of $P$, and let $\cQ'$ be the
distribution obtained by drawing $q \sim \frac{1}{n}(Q_1, \dotsc, Q_n)$ in the same way. Let $\rho =
\Pr{P > 0}, \gamma = \Pr{Q > 0}$. Write $S$ for
the set of vectors $p$ such that $|1 - \sum_i p_i| \leq \nu$ and $| \#\supp(p) - n\rho |
< \alpha \delta n$, and write $T$ for the set of vectors $q$ such that $|1-\sum_i q_i| \leq \nu$ and
$|\#\supp(q) - n\gamma| \leq \alpha \delta n$.

We will define $\cP$ to be the distribution $\cP'$ conditioned on the event $S$, while $\cQ$ is the
distribution $\cQ'$ conditioned on $T$. Wu \& Yang \cite{WY19} show that these events occur with
high probability (in particular, $\cP,\cQ$ are well-defined). It is clear that for each $p \in S, q
\in T$, we will have
\[
  \#\supp(p) - \#\supp(q) \geq n\rho - n\gamma - 2\alpha\delta n = n\delta - 2\alpha\delta n =
(1-2\alpha)\delta n \,,
\]
as desired, and $p(x),q(x) \geq (1+\nu)/n$ for all $x \in \supp(p), \supp(q)$ respectively. So it
remains to show the bound on $\widetilde{\DEC}(\cP,\cQ)$.

Let the random variable $s(\cP)$ be the vector of values seen from a random $p \sim \cP$ by a
Poisson sampling algorithm with parameter $m$, i.e.~for $s=s(\cP)$ and $p \sim \cP, s_i \sim
\Poi(mp_i)$. Wu \& Yang \cite{WY19} prove that
\[
  \| s(\cP) - s(\cQ) \|_\TV \leq \frac{2\lambda}{n\nu^2} + \frac{2}{n\alpha^2 \delta^2}
    + n \left(\frac{em\lambda}{2nL}\right)^L \,,
\]
which by assumption is less than $1/3$. Therefore, if a Poisson sampling algorithm $A$ outputs $1$
with probability at least $2/3$ over the random variable $s(\cP)$, it will output 1 with probability
greater than $2/3 - 1/3 = 1/3$. Therefore $\widetilde{\DEC}(\cP,\cQ) \geq m$ as desired.
\end{proof}

\begin{lemma}[\cite{WY19}, Lemma 7]
\label{lemma:wy variables}
For any $L \in \bN$, $\nu > 0, \lambda > 1 + \nu$,
there exist random variables $P,Q$ such that:
\begin{enumerate}
\item $P,Q$ are supported on $\{0\} \cup [1+\nu,\lambda]$;
\item $\Ex{P}=\Ex{Q}=1$ and $\forall j \in [L], \Ex{P^j} = \Ex{Q^j}$; and,
\item For $t = \sqrt{\frac{1+\nu}{\lambda}}$,
\[
  \Pr{P > 0}-\Pr{Q>0} = \frac{(1+t)^2}{1+\nu} \cdot \left(1-\frac{2t}{1+t}\right)^L \,.
\]
\end{enumerate}
\end{lemma}

\begin{proof}[Proof of \cref{thm:wy adaptation}]
For any parameters $L \in \bN, \nu > 0, \lambda > 1+\nu$, we obtain from \cref{lemma:wy
variables} random variables $P,Q$ taking values in $\{0\} \cup [1+\nu,\lambda]$ such that
$\Ex{P}=\Ex{Q}=1, \Ex{P^j}=\Ex{Q^j}$ for all $j \in [L]$, and
\[
  \Pr{P > 0} - \Pr{Q > 0} = \frac{(1+t)^2}{1+\nu}\cdot \left(1-\frac{2t}{1+t}\right)^L =: \epsilon
  \,,
\]
where $t = \sqrt{\frac{1+\nu}{\lambda}}$. Then \cref{lemma:wy poisson lower bound} implies that
for any $m,n \in \bN$ and $\alpha > 0$, we get distributions $\cP,\cQ$ over $\cD_n(\nu)$ such that
$\widetilde{\DEC}(\cP,\cQ) \geq m$ and $\#\supp(p)-\#\supp(q) \geq (1-2\alpha)\epsilon n$ for all $p
\in \supp(\cP), q \in \supp(\cQ)$, as long
as
\begin{equation}
\label{eq:lb constraint}
  \frac{2\lambda}{n\nu^2} + \frac{2}{n\alpha^2\epsilon^2} + n\left(\frac{em\lambda}{2nL}\right)^L
  < 1/3 \,.
\end{equation}
Suppose that $(1-2\alpha)\epsilon = 1-\delta$. Then for all $p \in \supp(\cP)$ we will have
$\#\supp(p) \geq (1-2\alpha)\epsilon n = (1-\delta)n$, as desired, and for all $q \in \supp(\cQ)$
we will have $\#\supp(q) \leq n - (1-2\alpha)\epsilon n = \delta n$. For any $p$ we also have
$p_i \geq (1+\nu)/n$ so the normalized distribution $\cP'$ defined in
\cref{lemma:wy poissonization} will have densities $p_i \geq \frac{1+\nu}{n \sum_i p_i} \geq
\frac{1+\nu}{n(1+\nu)} = 1/n$, and the same for $\cQ'$. Then for any $\nu = o(1)$, we will obtain a
lower bound of
\[
  \DEC(\cP',\cQ') = \Omega(\widetilde{\DEC}(\cP',\cQ')) = \Omega(m) \,.
\]
Therefore, what remains is to prove \cref{eq:lb constraint} with parameter $\epsilon =
\frac{1-\delta}{1-2\alpha}$.

Wu \& Yang (\cite{WY19}, equation 34) show that for sufficiently large constant $C$, if $\delta =
1-(1-2\alpha)\epsilon > C\frac{\sqrt{\log n}}{n^{1/4}}$ and $(1-2\alpha)\epsilon \geq n^{-o(1)}$
(where the latter holds trivially in our case because $(1-2\alpha)\epsilon \geq 1-\delta > 1/2$),
then there are parameters such that $\nu = o(1)$ and \cref{eq:lb constraint} holds with
\[
m = \Omega\left(\frac{n}{\log n}\log^2\frac{1}{(1-2\alpha)\epsilon}\right)
  = \Omega\left(\frac{n}{\log n}\log^2\frac{1}{1-\delta}\right) \,,
\]
which proves the theorem.
\end{proof}

\begin{center}
\textbf{\large Acknowledgments}
\end{center}
Thanks to Cameron Seth for helpful discussions and questions.

\bibliographystyle{alpha}
\bibliography{references.bib}
\end{document}